\documentclass[11pt]{article}
\usepackage[cm]{fullpage}
\usepackage{chngpage}
\usepackage{amsthm, amsmath, amsgen,amstext, amsfonts}
\usepackage{amsfonts,amsgen,amstext,dsfont}
\usepackage{amsthm, amsmath, amsgen,amstext, amsfonts}
\usepackage{amsfonts,amsgen,amstext,dsfont}

\usepackage{multirow,ulem}
\usepackage{algorithm,url,listings,wrapfig}
\usepackage[usenames,dvipsnames]{color}
\usepackage{graphicx,xspace,cite,tabularx,multibib}
\usepackage[noend]{algorithmic}

\newcommand{\rev}[1]{{#1}}
\newcommand{\edit}[1]{}

\newtheorem{definition}{Definition}[section]
\newtheorem{proposition}{Proposition}[section]
\newtheorem{theorem}{Theorem}[section]

\newcommand{\felix}{\textsc{Felix}\xspace}

\newcommand{\cimple}{\textsc{Cimple}\xspace}
\newcommand{\syst}{\textsc{SystemT}\xspace}

\newcommand{\tuffy}{\textsc{Tuffy}\xspace}

\newcommand{\mln}{\textsc{MLN}\xspace}
\newcommand{\mlns}{\textsc{MLN}s\xspace}
\newcommand{\alc}{\textsc{Alchemy}\xspace}
\newcommand{\eat}[1]{}

\newcommand{\set}[1]{\ensuremath \left\{#1\right\}}
\newcommand{\rel}[1]{\mathtt{#1}}
\newcommand{\con}[1]{`#1'}
\newcommand{\imp}{=>}

\newcommand{\la}{\leftarrow}
\newcommand{\cost}{\mathrm{cost}_\mathrm{mln}}
\newcommand{\costclass}{\mathrm{cost}_\mathrm{class}}

\newcommand{\costcoref}{\mathrm{cost}_\mathrm{coref}}

\newcommand{\R}{\mathds{R}}
\def\compactifytwo{\itemsep=-2pt \topsep=-1pt \partopsep=-2pt \parsep=-2pt}
\let\latexusecounter=\usecounter

\let\latexusecounter=\usecounter

\newcounter{examplecounter}[section]
\newenvironment{example}                        
  {\refstepcounter{examplecounter}\trivlist\item
    [\hskip\labelsep{\bf Example \theexamplecounter}]}
  {\endtrivlist}                                

\let\oldmarginpar\marginpar
\renewcommand\marginpar[1]{\-\oldmarginpar[\raggedleft\footnotesize #1]%
{\raggedright\footnotesize #1}}

\newcites{app}{References}

\begin{document}
\normalem
\title{Scaling Inference for Markov Logic with a Task-Decomposition Approach}
\author{
Feng Niu
\and
Ce Zhang
\and
Christopher R\'{e}
\and
Jude Shavlik
\and
\begin{minipage}{\textwidth}
  \centering
University of Wisconsin-Madison\\
\{leonn, czhang, chrisre, shavlik\}@cs.wisc.edu
\end{minipage}
}

\maketitle
\begin{abstract}
Motivated by applications in large-scale knowledge base construction,
we study the problem of scaling up a sophisticated statistical
inference framework called Markov Logic Networks (MLNs). Our approach,
Felix, uses the idea of Lagrangian relaxation from mathematical
programming to decompose a program into smaller tasks while
preserving the joint-inference property of the original MLN. The
advantage is that we can use highly scalable specialized algorithms
for common tasks such as classification and coreference. We propose an
architecture to support Lagrangian relaxation in an RDBMS which we
show enables scalable joint inference for MLNs. We
empirically validate that Felix is significantly more scalable and
efficient than prior approaches to MLN inference by constructing a
knowledge base from 1.8M documents as part of the TAC challenge. We
show that Felix scales and achieves state-of-the-art quality numbers.
In contrast, prior approaches do not scale even to a subset of the
corpus that is three orders of magnitude smaller.
\end{abstract}

\section{Introduction}
\label{sec:intro}

Building large-scale knowledge bases from text has recently received
tremendous interest from academia~\cite{weikum2010information}, e.g.,
CMU's NELL~\cite{carlson2010toward}, MPI's
YAGO~\cite{kasneci2008yago,nakashole2011scalable}, and from industry,
e.g., Microsoft's EntityCube~\cite{zhu2009statsnowball}, and IBM's
Watson~\cite{ferrucci2010building}. In their quest to extract
knowledge from free-form text, a major problem that all these systems
face is coping with inconsistency due to both conflicting information
in the underlying sources and the difficulty for machines to
understand natural language text. To cope with this challenge, each of
the above systems uses statistical inference to resolve these
ambiguities in a principled way.  To support this, the research
community has developed sophisticated statistical inference
frameworks, e.g., PRMs~\cite{friedman1999learning},
BLOG~\cite{milch2005blog}, MLNs~\cite{DBLP:journals/ml/RichardsonD06},
SOFIE~\cite{suchanek2009sofie}, Factorie~\cite{mccallum2009factorie},
and LBJ~\cite{rizzolo2010learning}. The key challenge with these
systems is efficiency and scalability, and to develop the next
generation of sophisticated text applications, we argue that a
promising approach is to improve the efficiency and scalability of the
above frameworks.

To understand the challenges of scaling such frameworks, we focus on
one popular such framework, called {\it Markov Logic Networks}
(\mlns), that has been successfully applied to many challenging text
applications~\cite{poon2007joint-ie,suchanek2009sofie,zhu2009statsnowball,andrzejewski2011framework}. In
Markov Logic one can write \rev{first-order logic rules with weights
(that intuitively model our confidence in a rule)}
\edit{was ``rules (syntactically similar to datalog or
SQL queries) with weights''; want to avoid misunderstanding about ``mixing
basic concepts of MLNs and datalog''}; this allows a developer to capture rules
that are likely, but not certain, to be correct. A key technical
challenge has been the scalability of \mln inference. Not
surprisingly, there has been intense research interest in techniques
to improve the scalability and performance of \mlns, such as improving
memory efficiency~\cite{singla2008lifted}, leveraging database
technologies~\cite{tuffy-vldb11}, and designing algorithms for
special-purpose
programs~\cite{andrzejewski2011framework,suchanek2009sofie}. Our work
here continues this line of work.

\begin{figure}[t]
\centering \includegraphics[width=0.45\textwidth]{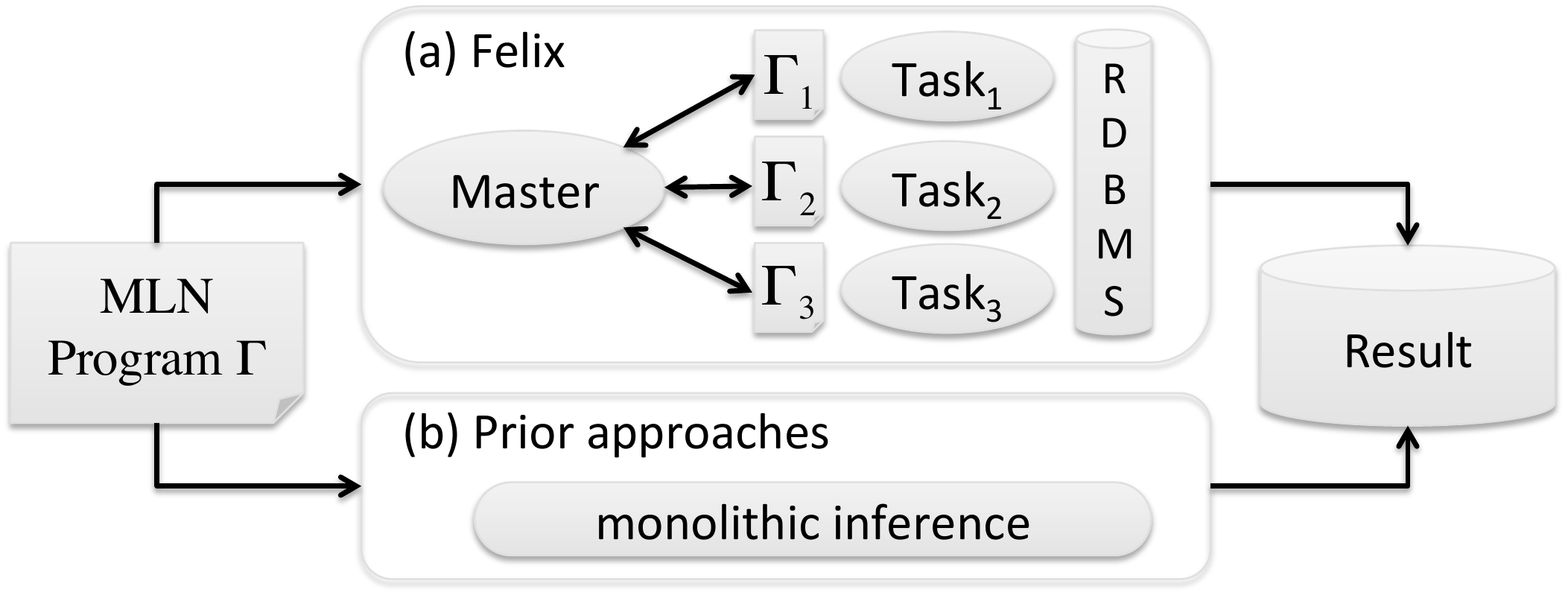}
\caption{\felix breaks an input
  program, $\Gamma$, into several, smaller tasks (shown in Panel
  a), while prior approaches are monolithic (shown in Panel b).}
  \label{fig:framework}
\end{figure}

Our goal is to use Markov Logic to construct a structured database of
facts and then answer questions like {\it ``which Bulgarian leaders
attended Sofia University and when?''}  with provenance from text.
(Our system, \felix, answers {\it Georgi Parvanov} and points to a handful of
sentences in a corpus to demonstrate its answer.)  During the
iterative process of constructing such a knowledge base from text and
then using that knowledge base to answer sophisticated questions, we have
found that it is critical to efficiently process structured queries
over large volumes of structured data. And so, we have built Felix on
top of an RDBMS. However, as we verify experimentally later in this
paper, the scalability of previous RDBMS-based solutions to \mln
inference~\cite{tuffy-vldb11} is still limited. Our key observation is
that in many text processing applications, one must solve a handful of
common subproblems, e.g., coreference resolution or
classification. Some of these have been studied for decades, and so
have specialized algorithms \rev{with} \edit{was 
``may be able to achieve'' which sounds somewhat uncertain} higher
scalability on these subproblems than the monolithic inference used by
typical Markov Logic systems. Thus, our goal is to leverage the
specialized algorithms for these subproblems to provide more scalable
inference for general Markov Logic programs in an
RDBMS. Figure~\ref{fig:framework} illustrates the difference at a high
level between \felix and prior
approaches: prior approaches, such as
\alc~\cite{DBLP:journals/ml/RichardsonD06} or
\tuffy~\cite{tuffy-vldb11}, are monolithic in that they 
attack the entire MLN inference problem with one algorithm; in
constrast, \felix \rev{decomposes} \edit{was ``will decompose''} the problem into several
small tasks.

To achieve this goal, we observe that the problem of inference in
an \mln -- and essentially any kind of statistical inference -- can be
cast as a mathematical optimization problem. Thus, we adapt techniques
from the mathematical programming literature to \mln inference. In
particular, we consider the idea of \emph{Lagrangian
relaxation}~\cite[p.~244]{Bertsekas:book} that allows one to decompose
a complex optimization problem into multiple pieces that are hopefully
easier to
solve~\cite{wolsey1998integer,DBLP:conf/emnlp/RushSCJ10}. Lagrangian
relaxation is a widely deployed technique to cope with many difficult
mathematical programming problems, and it is the theoretical
underpinning of many state-of-the-art inference algorithms for
graphical models, e.g., \edit{removed ``Viterbi,''} Belief
Propagation~\cite{Wainwright:2008:GME:1523420}. In many -- but not all
-- cases, a Lagrangian relaxation has the same optimal solution as the
underlying original
problem~\cite{Bertsekas:book,boyd:cvx,wolsey1998integer}. At a high
level, Lagrangian relaxation gives us a message-passing protocol that
resolves inconsistencies among conflicting predictions to
accomplish \emph{joint-inference}. Our system, \felix, does not
actually construct the mathematical program, but uses Lagrangian
relaxation as a formal guide to decompose an \mln program into
multiple tasks and construct an appropriate message-passing scheme.

Our first technical contribution is an architecture to scalably
perform \mln inference in an RDBMS using Lagrangian relaxation. Our
architecture models each subproblem as a \emph{task} that takes as
input a set of relations, and outputs another set of relations.  For
example, our prototype of \felix implements specialized algorithms for
classification and coreference resolution (coref); these tasks
frequently occur in text-processing applications. By modeling tasks in
this way, we are able to use SQL queries for all {\em data movement}
in the system: both transforming the input data into an appropriate
form for each task and encoding the message passing of Lagrangian
relaxation between tasks. In turn, this allows \felix to leverage the
mature, set-at-a-time processing power of an RBDMS to achieve
scalability and efficiency. On all programs and datasets that we
experimented with, our approach converges rapidly to the optimal
solution of the Lagrangian relaxation.  Our ultimate goal is to build
high-quality applications, and we validate on several knowledge-base
construction tasks that \felix achieves higher scalability and
essentially identical result quality compared to prior \mln
systems. More precisely, when prior \mln systems are able to
scale, \felix converges to the same quality (and sometimes more
efficiently). When prior \mln systems fail to scale, \felix can still
produce high-quality results. We take
this as evidence that \felix's approach is a promising direction to
scale up large-scale statistical inference.  Furthermore, we validate
that being able to integrate specialized algorithms is crucial
for \felix's scalability: after disabling specialized algorithms,
\felix no longer scales to the same datasets.

Although the RDBMS provides some level of scalability for data movement
inside \felix, the scale of data passed between tasks (via SQL
queries) may be staggering. The reason is that statistical algorithms
may produce huge numbers of combinations (say all pairs of potentially
matching person mentions).
The sheer sizes of intermediate results are often
killers for scalability, e.g., the complete input to coreference
resolution on an Enron dataset\eat{\footnote{\scriptsize{\url{http://bailando.sims.berkeley.edu/enron_email.html}}}} has $1.2\times10^{11}$
tuples.
The saving grace is that a task may access the intermediate data in an
on-demand manner.  For example, a popular coref algorithm repeatedly
asks {\it ``given a fixed word $x$, tell me all words that are likely
to be coreferent with
$x$.''}~\cite{ailon2008aggregating,arasu2009large}.  Moreover, the
algorithm only asks for a small fraction of such $x$.  Thus, it would
be wasteful to produce all possible matching pairs. Instead we can
produce only those words that are needed on-demand (i.e., materialize
them lazily).
\felix considers a richer space of possible materialization strategies
than simply eager or lazy: it can choose to eagerly materialize one or
more subqueries responsible for data movement between
tasks~\cite{ramakrishnan1995survey}.  To make such decisions,
\felix's second contribution is a novel cost model that leverages 
the cost-estimation facility in the RDBMS coupled with the data-access
patterns of the tasks. On the Enron dataset, our cost-based approach
finds execution plans that achieve two orders of magnitude speedup
over eager materialization and 2-3X speedup compared to lazy materialization.

\rev{Although \felix allows a user to provide any decomposition scheme,} \edit{add}
identifying decompositions could be difficult for some users, so we do
not want to force users to specify a decomposition to use Felix.  To support
this, we need a compiler that performs task decomposition given a
standard MLN program as input.  Building on classical and new results
in embedded dependency inference from the database theory
literature~\cite{abiteboul1995foundations,
DBLP:conf/sigmod/AbiteboulH88,gmsv93:surajits:paper,Fan:VLDB2008}, we
show that the underlying problem of compilation is
$\Pi_2 \mathbf{P}$-complete in easier cases, and undecidable in more
difficult cases.  To cope, we develop a sound (but not complete)
compiler that takes as input an ordinary \mln program, identifies
common tasks such as classification and coref, and then assigns those
tasks to specialized algorithms.

To validate that our system can perform sophisticated knowledge-base
construction tasks, we use the \felix system to implement a solution for the 
TAC-KBP (Knowledge Base Population)
challenge.\footnote{\scriptsize{\url{http://nlp.cs.qc.cuny.edu/kbp/2010/}}}
Given a 1.8M document corpus, the goal is to perform two related
tasks: (1) {\em entity linking}: extract all entity mentions and map
them to entries in Wikipedia, and (2) {\em slot filling}: determine
relationships between entities. The reason for choosing this task is
that it contains ground truth so that we can assess the results: We
achieved F1=0.80 on entity linking (human performance is 0.90), and F1=0.34
on slot filling (state-of-the-art quality).{\footnote{\scriptsize F1 is the
harmonic mean of precision and recall.}}
In addition to
KBP, we also use three information extraction (IE) datasets that
have state-of-the-art solutions. On all four datasets, we show that
\felix is significantly more scalable than monolithic systems
such as \tuffy and \alc; this in turn enables \felix to efficiently
process sophisticated \mlns and produce high-quality results.
Furthermore, we validate that our individual technical contributions
are crucial to the overall performance and quality of \felix.

\begin{figure*}[!t]
\centering
{\scriptsize
\begin{tabular}{@{\hspace{0cm}}c@{\hspace{0cm}}||@{\hspace{0cm}}c@{\hspace{-0.2cm}}||c}
\begin{tabular}{l}
$\rel{pSimHard}$(per1, per2)\\
$\rel{pSimSoft}$(per1, per2)\\
$\rel{oSimHard}$(org1, org2)\\
$\rel{pSimSoft}$(org1, org2)\\
$\rel{coOccurs}$(per, org)\\
$\rel{homepage}$(per, page)\\
$\rel{oMention}$(page, org)\\
$\rel{faculty}$(org, per)\\
$\rel{*affil}$(per, org)\\
$\rel{*oCoref}$(org1, org2)\\
$\rel{*pCoref}$(per1, per2)
\end{tabular}
\hspace{5pt}
&

\begin{tabular}{l}
$\rel{coOccurs}$(\con{Ullman}, \con{Stanford Univ.})\\
$\rel{coOccurs}$(\con{Jeff Ullman}, \con{Stanford})\\
$\rel{coOccurs}$(\con{Gray}, \con{San Jose Lab})\\
$\rel{coOccurs}$(\con{J. Gray}, \con{IBM San Jose})\\
$\rel{coOccurs}$(\con{Mike}, \con{UC-Berkeley})\\
$\rel{coOccurs}$(\con{Mike}, \con{UCB})\\
$\rel{coOccurs}$(\con{Joe}, \con{UCB})\\
$\rel{faculty}$(\con{MIT}, \con{Chomsky})\\
$\rel{homepage}$(\con{Joe}, \con{Doc201})\\
$\rel{oMention}$(\con{Doc201}, \con{IBM})\\
$\cdots$
\end{tabular}
\hspace{5pt}
&
\begin{tabular}{c@{\hspace{0pt}}l@{\hspace{1pt}}l@{\hspace{0pt}}}
\textbf{weight} & \parbox{40pt}{\centering \textbf{rule}} & \\
$+\infty$ & $\rel{pCoref}(p,p)$ & ($F_1$)\\
$+\infty$ & $\rel{pCoref}(p1,p2)\imp\rel{pCoref}(p2,p1)$ & ($F_2$)\\
$+\infty$ & $\rel{pCoref}(x,y), \rel{pCoref}(y,z)\imp\rel{pCoref}(x,z)$ &
($F_3$)\\
6 & $\rel{pSimHard}(p1,p2) \imp \rel{pCoref}(p1,p2)$ & ($F_4$)\\
2 & $\rel{affil}(p1,o), \rel{affil}(p2,o),\rel{pSimSoft}(p1,p2)$ $\imp \rel{pCoref}(p1,p2)$ & ($F_5$)
\vspace{3pt}\\
\hline
$+\infty$ & $\rel{faculty}(o,p) \imp \rel{affil}(p,o)$  & ($F_{6}$)\\
8 & $\rel{homepage}(p,d), \rel{oMention}(d,o) \imp \rel{affil}(p,o)$ &
($F_{7}$)\\
3 & $\rel{coOccurs}(p,o1), \rel{oCoref}(o1,o2) \imp \rel{affil}(p,o2)$ &
($F_{8}$)\\
4 & $\rel{coOccurs}(p1,o), \rel{pCoref}(p1,p2) \imp \rel{affil}(p2,o)$ &
($F_{9}$)\\
\vspace{3pt}
 $\ldots$ &
\end{tabular} \\

\hline
\textbf{Schema} &  \textbf{Evidence} & \textbf{Rules}
\end{tabular}}
\caption{An example \mln program that performs three tasks jointly:
1. discover affiliation relationships between people and organizations ($\rel{affil}$);
2. resolve coreference among people mentions ($\rel{pCoref}$); and
3. resolve coreference among organization mentions ($\rel{oCoref}$).
The remaining eight relations are evidence relations. In particular, 
$\rel{coOccurs}$ stores person-organization co-occurrences;
$*\rel{Sim}*$ relations are string similarities.
}
\label{fig:mlns}
\end{figure*}

\paragraph*{Outline}
In Section~\ref{sec:related}, we describe related work. 
In Section~\ref{sec:prelim}, we describe a simple text application encoded as an \mln program,
and the Lagrangian relaxation technique in mathematical programming. 
In Section~\ref{sec:arch}, we present an overview of \felix's architecture and 
some key concepts. In Section~\ref{sec:tech}, we describe
key technical challenges and how \felix addresses them: how to
execute individual tasks with high performance and quality, 
how to improve the data movement efficiency between tasks, and how
to automatically recognize specialized tasks in an \mln program. 
In Section~\ref{sec:experiments}, we use extensive experiments to validate the overall
advantage of \felix as well as individual technical contributions.

\section{Related Work}
\label{sec:related}
There is a trend to build semantically deep text applications  
with increasingly sophisticated statistical inference~\cite{zhu2009statsnowball,
  weld2009using,suchanek2009sofie,fang2011searching}. We follow on this line of
work. However, while the goal of prior work is to
explore the effectiveness of different correlation structures on
particular applications, our goal is to support general application
development by scaling up existing statistical inference frameworks.
Wang et al.~\cite{wang2011hybrid} explore multiple inference algorithms
for information extraction. However, their system focuses on managing
low-level extractions in CRF models, whereas our goal is to use \mln
to support knowledge base construction.

\felix specializes to \mlns. There are, however, other
statistical inference frameworks such as PRMs~\cite{friedman1999learning},
BLOG~\cite{milch2005blog}, \rev{Factorie~\cite{mccallum2009factorie,wick2010scalable}} \edit{added Wick's VLDB 2010 paper~\cite{wick2010scalable}}, and
PrDB~\cite{sen09-PrDB}. Our hope is that the techniques developed
here apply to these frameworks as well.

Researchers have proposed different approaches to improving MLN
inference performance in the context of text applications.
In StatSnowball~\cite{zhu2009statsnowball}, Zhu et al. demonstrate
high quality results of an MLN-based approach. To address the
scalability issue of generic \mln inference, they make additional
independence assumptions in their programs. In contrast, the goal of
\felix is to automatically scale up statistical inference while sticking to
\mln semantics. Theobald et
al.~\cite{theobald2010urdf} design specialized MaxSAT algorithms that
efficiently solve \mln programs of special forms. In contrast, we study how
to scale general \mln programs.
\rev{Riedel~\cite{riedel09cutting} proposed a cutting-plane meta-algorithm
that iteratively performs grounding and inference, but the underlying
grounding and inference procedures are still for generic \mlns.} \edit{add}
In Tuffy~\cite{tuffy-vldb11}, the authors improve the scalability of
\mln inference with an RDBMS, but their system is still a monolithic
approach that consists of generic inference procedures.

As a classic technique, Lagrangian relaxation has been applied to
closely related statistical models (i.e., graphical 
models)~\cite{DBLP:journals/corr/abs-0710-0013,Wainwright:2008:GME:1523420}.
However, there the input is directly a mathematical optimization problem
and the granularity of decomposition is individual variables. In contrast,
our input is a program in a high-level language,
and we perform decomposition at the relation level
inside an RDBMS.

Our materialization tradeoff strategy is related to view
materialization and
selection~\cite{shukla1998materialized,chirkova2006answering} in the
context of data warehousing.  However, our problem setting is
different: we focus on batch processing so that we do not
consider maintenance cost. The idea of lazy-eager tradeoff in view
materialization or query answering has also been applied to
probabilistic databases~\cite{wick2010scalable}. However, their goal is
efficiently maintaining intermediate results, rather than choosing
a materialization strategy. Similar in spirit to our approach is
Sprout~\cite{olteanu2009sprout}, which considers lazy-versus-eager
plans for when to apply confidence computation, but they do not consider
inference decomposition.

\eat{
The XLog framework~\cite{DBLP:conf/vldb/ShenDNR07} allows black boxes
(e.g., Perl scripts) to be called from datalog programs (simulating
table functions). Our approach differs in many key aspects: (1) \felix
uses a single language (Markov Logic) and so our tasks are not
black boxes: rather they have a formal semantics; (2) we discover
specialized tasks automatically; and (3) we optimize data movement.
}

\section{Preliminaries}
\label{sec:prelim}
To illustrate how \mlns can be used in text-processing applications,
we first walk through a program that extracts affiliations between people and organizations
from Web text. We then
describe how Lagrangian relaxation is used for mathematical optimization.

\subsection{Markov Logic Networks in Felix} \label{sec:mln-syntax}
In text applications, a typical first step is to use standard NLP toolkits
to generate raw data such as plausible mentions of
people and organizations in a Web corpus and their co-occurrences.
But transforming such raw signals into high-quality and semantically
coherent knowledge bases is a challenging task.
For example, a major challenge is that a single real-world
entity may be referred to in many different ways, e.g., {\it``UCB''}
and {\it ``UC-Berkeley''}.
To address such challenges, \mlns provide a framework where we
can express logical assertions that are only likely to be true 
(and quantify such likelihood). Below we explain the key concepts
in this framework by walking through an example.

Our system \felix is a middleware system: it takes as input a standard
\mln program, performs statistical inference, and outputs its results
into one or more relations that are stored in a relational 
database (PostgreSQL). An \mln program consists of three parts: \emph{schema}, 
\emph{evidence}, and \emph{rules}. 
To tell \felix what data will be provided or generated, the
user provides a \emph{schema}.  Some relations are standard database
relations, and we call these relations \emph{evidence}.  Intuitively,
evidence relations contain tuples that we assume are correct. In the
schema of Figure~\ref{fig:mlns}, the first eight relations are
evidence relations. For example, we know that \con{Ullman} and
\con{Stanford Univ.} co-occur in some webpage, and that \con{Doc201}
is the homepage of \con{Joe}.\eat{Other evidence includes string
similarity information.}
In addition to evidence relations, there are
also relations whose content we do not know, but we want the \mln
program to predict; they are called {\it query relations}. In
Figure~\ref{fig:mlns}, $\rel{affil}$ is a query relation since we want
the \mln to predict affiliation relationships between persons and
organizations. The other two query relations are $\rel{pCoref}$ and
$\rel{oCoref}$, for person and organization coreference, respectively.

In addition to schema and evidence, we also provide a set of \mln
rules that encode our knowledge about the correlations and constraints
over the relations. An \mln rule is a first-order logic formula
associated with an extended-real-valued number called a {\it
weight}. Infinite-weighted rules are called hard rules, which means
that they must hold in any prediction that the \mln system makes. In
contrast, rules with finite weights are soft rules: a positive weight
indicates confidence in the rule's correctness.\footnote{\scriptsize{Roughly these weights
correspond to the log odds of the probability that the statement is
true. (The log odds of probability $p$ is $\log\frac{p}{1-p}$.) In
general, these weights do not have a simple probabilistic
interpretation~\cite{DBLP:journals/ml/RichardsonD06}.}} (In \felix, weights
can be set by the user or automatically learned. We do not discuss
learning in this work.)

\begin{example}
An important type of hard rule is a standard SQL query, e.g., to
transform the results for use in the application. A more sophisticated
example of hard rule is to encode that coreference has a transitive
property, which is captured by the hard rule $F_3$. Rules $F_{8}$ and
$F_{9}$ use person-organization co-occurrences ($\rel{coOccurs}$)
together with coreference ($\rel{pCoref}$ and $\rel{oCoref}$) to
deduce affiliation relationships ($\rel{affil}$). These rules are soft
since co-occurrence in a webpage does not necessarily imply
affiliation.
\end{example}

Intuitively, when a soft rule is violated, we pay a \emph{cost} equal
to the absolute value of its weight (described below).  For example, if
$\rel{coOccurs}$(\con{Ullman}, \con{Stanford Univ.}) and
$\rel{pCoref}$(\con{Ullman}, \con{Jeff Ullman}), but not
$\rel{affil}$(\con{Jeff Ullman}, \con{Stanford Univ.}), then we pay a
cost of 4 because of $F_9$. The goal of an \mln inference algorithm is to
find a prediction that minimizes the sum of such costs.

\eat{
Similarly, affiliation relationships can be used to deduce non-obvious
coreferences. For instance, using the fact that \con{Mike} is
affiliated with both \con{UC-Berkeley} and \con{UCB}, \felix may
infer that \con{UC-Berkeley} and \con{UCB} refer to the same
organization (rules on $\rel{oCoref}$ are omitted from
Figure~\ref{fig:mlns}).  If \felix knows that \con{Joe} co-occurs
with \con{UCB}, then it is able to infer Joe's
affiliation with \con{UC-Berkeley}. Such interdependencies between
predictions are sometimes called joint inference.
}

\paragraph*{Semantics}
An \mln program defines a probability distribution over database
instances (possible worlds).  Formally, we first fix a schema $\sigma$
(as in Figure~\ref{fig:mlns}) and a domain $D$. Given as input a set
of formulae $\bar F = F_1, \dots, F_N$ with weights $w_1, \dots, w_N$,
they define a probability distribution over {\em possible worlds}
(deterministic databases) as follows. Given a formula $F_k$ with free
variables $\bar x = (x_1,\cdots,x_m)$, then for each $\bar d \in
D^{m}$, we create a new formula $g_{\bar d}$ called a {\em ground
formula} where $g_{\bar d}$ denotes the result of substituting each
variable $x_i$ of $F_k$ with $d_i$. We assign the weight $w_k$ to
$g_{\bar d}$.  Denote by $G=(\bar g, w)$ the set of all such weighted ground
formulae of $\bar F$. We call the set of all tuples in $G$ the {\em
ground database}. Let $w$ be a function that maps each ground formula
to its assigned weight. Fix an \mln $\bar{F}$, then for any possible
world (instance) $I$ we say a ground formula $g$ is {\it violated} if
$w(g) > 0$ and $g$ is false in $I$, or if $w(g) < 0$ and $g$ is true
in $I$. We denote the set of ground formulae violated in a world $I$
as $V(I)$. The cost of the world $I$ is
\begin{equation}
\cost(I) = \sum_{g \in V(I)} |w(g)|
\label{eq:cost}
\end{equation}
Through $\cost$, an \mln defines a probability distribution over all
instances using the exponential family of distributions (that are the
basis for graphical models~\cite{Wainwright:2008:GME:1523420}):
\[ \Pr[I] = Z^{-1} \exp \set{ - \cost(I) }  \]
where $Z$ is a normalizing constant.

\paragraph*{Inference}
There are two main types of inference with \mlns: {\it MAP (maximum a
  posterior) inference}, where we want to find a most likely world,
i.e., a world with the lowest cost, and {\it marginal inference},
where we want to compute the marginal probability of each unknown
tuple.
Both types of inference are essentially mathematical optimization problems
that are intractable, and so existing \mln systems
implement generic (search/sampling) algorithms for inference.
As a baseline, \felix implements generic algorithms for both types of inference
as well.
Although \felix supports both types of inference in our decomposition
architecture, in this work we focus on MAP inference to simplify the presentation.

\subsection{Lagrangian Relaxation}
\label{sec:lagrangian}
We illustrate the basic idea of \emph{Lagrangian relaxation} with a simple example.
Consider the problem of minimizing a real-valued function $f(x_1,x_2,x_3)$.
Lagrangian relaxation is a technique that allows us to divide and conquer a problem
like this. For example, suppose that $f$ can be written as
\[
f(x_1,x_2,x_3) = f_1({x_1}, {x_2}) + f_2({x_2}, {x_3}).
\]
While we may be able to solve each of $f_1$ and $f_2$ efficiently,
that ability does not directly lead to a solution to $f$ since
$f_1$ and $f_2$ share the variable ${x_2}$.
However, we can rewrite $\min_{x_1,x_2,x_3} f(x_1,x_2,x_3)$ into the form
\[
\min_{{x_1,x_{21},x_{22},x_3}} f_1({x_1}, {x_{21}}) + f_2({x_{22}}, {x_3})
\text{ s.t. } {x_{21}}={x_{22}},
\]
where we essentially made two copies of ${x_2}$ and enforce
that they are identical. The significance of such rewriting is that we can
apply Lagrangian relaxation to the equality constraint to decompose the
formula into two independent pieces. To do this, we introduce a scalar
variable $\lambda\in\R$ (called a \emph{Lagrange multiplier}) and define
\begin{align*}
g(\mathbf{\lambda}) = \min_{x_1,x_{21},x_{22},x_3}
f_1({x_1}, {x_{21}}) + f_2({x_{22}}, {x_3}) + 
{\lambda}({x_{21}}-{x_{22}})
\end{align*}
Then $\max_{\lambda}g(\lambda)$ is called the \emph{dual problem} of
the original minimization problem on $f$. Intuitively, The dual
problem trades off a penalty for how much the copies $x_{21}$ and
$x_{22}$ disagree with the original objective value. If the resulting
solution of this dual problem is feasible for the original program
(i.e., satisfies the equality constraint), then this solution is also
an optimum of the original program~\cite[p.~168]{wolsey1998integer}.

The key benefit of such relaxation is that, instead of a single problem
on $f$, we can now compute $g(\lambda)$ by solving
two independent problems (each problem is grouped by parentheses)
that are hopefully (much) easier:
\[
g(\lambda) = 
\left(\min_{x_1,x_{21}} f_1(x_1,x_{21}) + \lambda x_{21} \right) + 
\left(\min_{x_{22},x_3} f_2(x_{22},x_3) - \lambda x_{22}\right).
\]
To compute $\max_{\lambda}g(\lambda)$, we can use
standard techniques such as
\emph{gradient descent}~\cite[p.~174]{wolsey1998integer}.

Notice that Lagrangian relaxation
could be used for \mln inference:
consider the case where $x_i$ are truth values of database tuples
representing a possible world $I$ and define $f$ to be $\cost(I)$
as in Equation~\ref{eq:cost}.
(\felix can handle marginal inference with Lagrangian relaxation
as well, but we focus on MAP inference to simplify presentation.)

\paragraph*{Decomposition Choices}
The Lagrangian relaxation technique leaves open the question of
\emph{how} to decompose a function $f$ in general and introduce equality constraints. 
These are the questions we need to answer first and foremost if we
want to apply Lagrangian relaxation to \mlns. Furthermore, 
it is important that we can scale up the execution of the 
decomposed program on large datasets.

\section{Architecture of Felix}
\label{sec:arch}

\begin{figure}
\centering
  \includegraphics[width=0.48\textwidth]{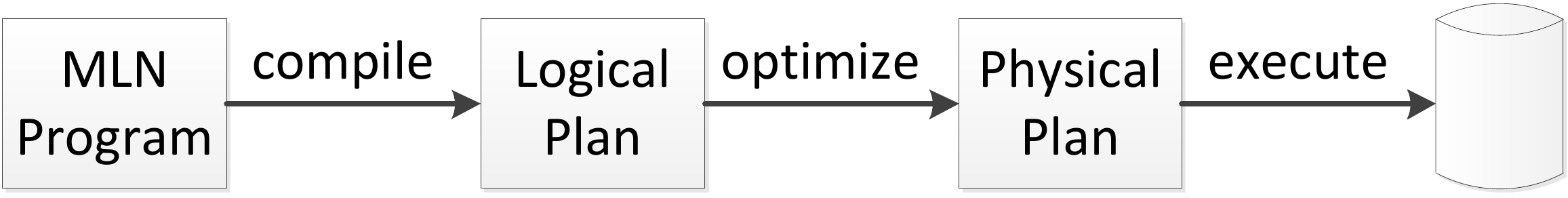}
\caption{Execution Pipeline of \felix.}
\label{fig:exec:pipe}
\end{figure}
In this section, we provide an overview of the \felix architecture and some
key concepts. We expand on further technical details in the next section.
At a high level, the way \felix performs \mln inference resembles
how an RDBMS performs SQL query evaluation: given an \mln
program $\Gamma$, \felix transforms it in
several phases as illustrated in Figure~\ref{fig:exec:pipe}: \felix
first {\em compiles} an MLN program into a {\em logical plan}
of tasks. Then, \felix performs {\em optimization}
(code selection) to select the best \emph{physical plan} that consists
of a sequence of
statements that are then executed (by a process called the \emph{Master}). In turn,
the Master may call an RDBMS or
statistical inference algorithms.

\subsection{Compilation}
\label{sec:arch:compilation}
In \mln inference, a variable of the underlying optimization problem 
corresponds to the truth value (for MAP inference)
or marginal probability (for marginal inference) of a query relation tuple.
While Lagrangian relaxation allows us to decompose an inference problem in arbitrary ways,
\felix focuses on decompositions at the level of relations: 
\felix ensures that an entire relation is either shared between subproblems or
exclusive to one subproblem.
A key advantage of this is that \felix can benefit from the set-oriented processing
power of an RDBMS. Even with this restriction, any
partitioning of the rules in an \mln program $\Gamma$ is a valid decomposition.
(For the moment, assume that all rules are soft;
we come back to hard rules in Section~\ref{sec:arch:execution}.)

Formally, let $\Gamma=\{\phi_i\}$ be a set of \mln rules; denote by $\mathcal{R}$ the set of
query relations and $\mathbf{x}_R$ the set of Boolean
variables (i.e., unknown truth values) of $R\in\mathcal{R}$.
Let $\Gamma_1,\ldots,\Gamma_k$ be a decomposition of $\Gamma$,
and $\mathcal{R}_i\subseteq\mathcal{R}$ the set of query relations
referred to by $\Gamma_i$. 
Define $\mathbf{x}_{\mathcal{R}}=\cup_{R\in\mathcal{R}}\mathbf{x}_R$;
similarly $\mathbf{x}_{\mathcal{R}_i}$. Then we can write the \mln cost
function as
\[
\min_{\mathbf{x}_{\mathcal{R}}}\cost^{\Gamma}(\mathbf{x}_{\mathcal{R}}) = \min_{\mathbf{x}_{\mathcal{R}}}\sum_{i=1}^{k}\cost^{\Gamma_i}(\mathbf{x}_{\mathcal{R}_i})
\]

To decouple the subprograms, we create a local copy of variables $\mathbf{x}^i_{\mathcal{R}_i}$
for each $\Gamma_i$, but also introduce Lagrangian multipliers 
$\mathbf{\lambda}_R^j\in\mathds{R}^{|\mathbf{x}_R|}$ for each $R\in\mathcal{R}$
and each $\Gamma_j$ s.t. $R\in\mathcal{R}_j$, resulting in the dual problem
\begin{align*}
&\max_{\lambda}g(\lambda)\\
\equiv&\max_{\lambda}
\left\{
\sum_{i=1}^{k}\min_{\mathbf{x}^i_{\mathcal{R}_i}}\left[\cost^{\Gamma_i}(\mathbf{x}^i_{\mathcal{R}_i})
+\mathbf{\lambda}_{\mathcal{R}_i}^i \cdot\mathbf{x}^i_{\mathcal{R}_i}
\right]
\right\}
\\
&\textrm{subject to}  \sum_{j:R\in\mathcal{R}_j}\mathbf{\lambda}_R^j=0\quad\forall R\in\mathcal{R}.
\end{align*}

Thus, to perform Lagrangian relaxation on $\Gamma$, we need to augment the cost
function of each subprogram with the $\mathbf{\lambda}_{\mathcal{R}_i}^i \cdot\mathbf{x}^i_{\mathcal{R}_i}$
terms. As illustrated in the example below, these additional terms are equivalent
to adding singleton rules with the multipliers as weights. As a result,
we can still solve the (augmented) subproblems $\Gamma^\lambda_i$ 
as \mln inference problems.

\begin{example}
\label{ex:happysad}
Consider a simple Markov Logic program\eat{\footnote{\scriptsize
We use this artificial example for notational simplicity.
}} $\Gamma$:\\

{\small
\begin{tabular}{llr}
$1$ & $\rel{GoodNews}(p) => \rel{Happy}(p)$ & $\phi_1$\\
$1$ & $\rel{BadNews}(p) => \rel{Sad}(p)$ & $\phi_2$ \\
$5$ & $\rel{Happy}(p) <=> \neg\rel{Sad}(p)$ & $\phi_3$\\
\end{tabular}
}\\

\noindent
where $\rel{GoodNews}$ and $\rel{BadNews}$ are evidence and 
the other two relations are queries. 
Consider the decomposition $\Gamma_1=\{\phi_1\}$ and $\Gamma_2=\{\phi_2,\phi_3\}$.
$\Gamma_1$ and $\Gamma_2$ share the relation $\rel{Happy}$;
so we create two copies of this relation: $\rel{Happy_1}$ and $\rel{Happy_2}$,
one for each subprogram.
To relax the need that $\rel{Happy_1}$ and $\rel{Happy_2}$ be equal,
we introduce Lagrange multipliers $\lambda_p$,
one for each possible tuple $\rel{Happy}(p)$. We thereby obtain a new program
$\Gamma^\lambda$:\\

{\small
\begin{tabular}{lll}
$1$ & $\rel{GoodNews}(p) => \rel{Happy_1}(p)$ & $\phi_1'$\\
$\lambda_p$ & $\rel{Happy_1}(p)$ & $\varphi_{1}$ \\
$1$ & $\rel{BadNews}(p) => \rel{Sad}(p)$ & $\phi_2$ \\
$5$ & $\rel{Happy_2}(p) <=> \neg\rel{Sad}(p)$ & $\phi_3'$\\
$-\lambda_p$ & $\rel{Happy_2}(p)$ & $\varphi_{2}$ \\
\end{tabular}
}\\

This program contains two subprograms,
$\Gamma^\lambda_1=\{\phi_1',\varphi_{1}\}$ and
$\Gamma^\lambda_2=\{\phi_2,\phi_3',\varphi_{2}\}$,
that can be solved independently.

\end{example}

\begin{figure}[t]
\centering
  \includegraphics[width=0.50\textwidth]{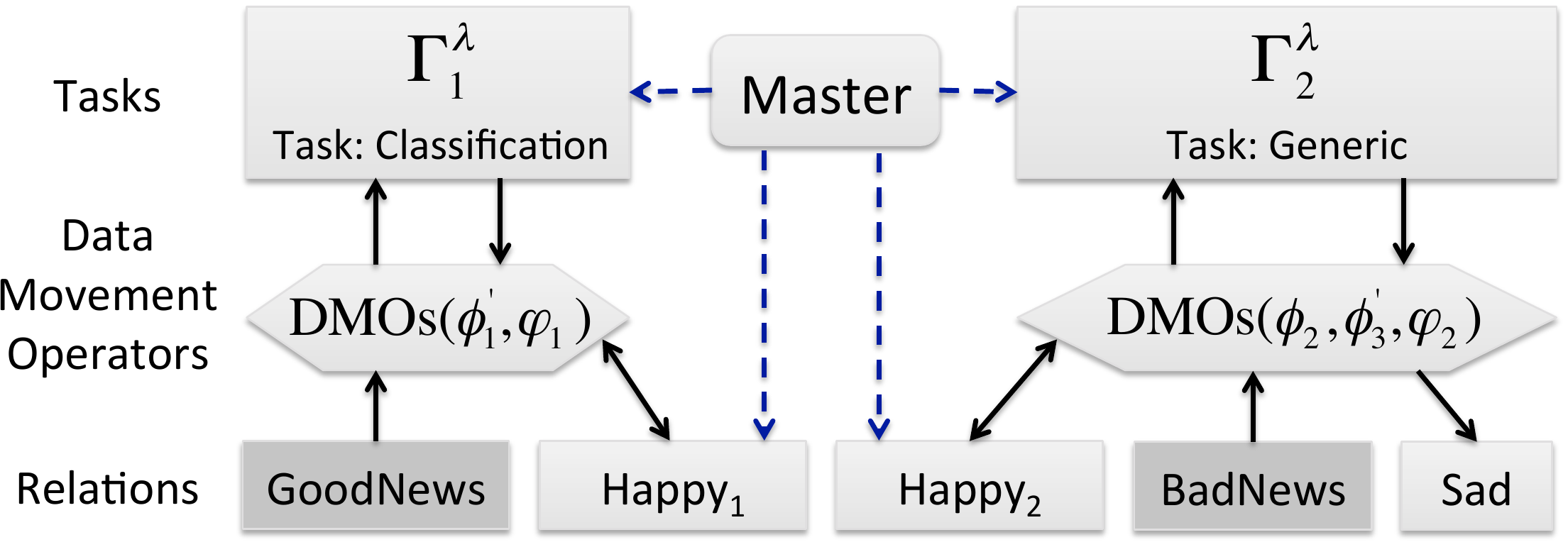}
  \caption{An example logical plan. 
Relations in shaded boxes are evidence relations.
Solid arrows indicate data flow; dash arrows are control.}
  \label{fig:logicalplan}
\end{figure}

The output of compilation is a \emph{logical plan} that consists of
a bipartite graph between a set of subprograms (e.g., $\Gamma^\lambda_i$) 
and a set of relations (e.g., $\rel{GoodNews}$ and $\rel{Happy}$).
There is an edge between a subprogram and a relation if
the subprogram refers to the relation.
In general, the decomposition
could be either user-provided or automatically generated.
In Sections~\ref{sec:compilation} we
discuss automatic decomposition.

\subsection{Optimization}
\label{sec:arch:optimization}
\label{sec:tasks}

The optimization stage fleshes out the logical plan with code
selection and generates a \emph{physical plan} with detailed \emph{statements} that
are to be executed by a process in \felix called the Master.
Each subprogram $\Gamma^\lambda_i$ in the logical plan is 
executed as a \emph{task} that encapsulates 
a statistical algorithm that consumes and produces relations.
The default algorithm assigned to each task is a generic
\mln inference algorithm that can handle any \mln program~\cite{tuffy-vldb11}.
However, as we will see in Section~\ref{sec:specoper}, there are
several families of \mlns that have specialized
algorithms with high efficiency and high quality.
For tasks matching those families, we execute them with corresponding
specialized algorithms. \eat{It is also possible for the user to provide
algorithms directly as long as the algorithm implements similar interfaces.}

The input/output relations of each task are not necessarily the relations
in the logical plan. For example, the input to a classification task could
be the results of some conjunctive queries translated from \mln rules.
To model such indirection, we introduce \emph{data movement operators} (DMOs),
which are essentially datalog queries that map between \mln relations
and task-specific relations. 
\rev{Roughly speaking, DMOs for specialized algorithms play a role that
is similar to what grounding does for generic \mln inference.} \edit{add}
Given a task $\Gamma^\lambda_i$, it is the responsibility
of the underlying algorithm to generate all necessary DMOs
and register them with \felix. 
Figure~\ref{fig:logicalplan} shows an enriched logical plan after 
code selection and DMO generation.
DMOs are critical to the performance of \felix, and so we need to
execute them efficiently.
We observe that the overall performance of an evaluation strategy for a DMO
depends on not only how well an RDBMS can execute SQL, but also \emph{how}
and \emph{how frequently} a task queries this DMO -- namely
the access pattern of this task.

To expose the access patterns of a task to \felix, we model DMOs as {\em adorned views}~\cite{ullman1985implementation}. 
In an adorned view, each variable in the head of a view definition is associated with a
binding-type, which is either $\mathsf{b}$ (bound) or $\mathsf{f}$
(free).  Given a DMO $Q$, denote by $\bar{x}^\mathsf{b}$
(resp. $\bar{x}^\mathsf{f}$) the set of bound (resp. free) variables
in its head. Then we can view $Q$ as a function mapping an
assignment to $\bar{x}^\mathsf{b}$ (i.e., a tuple) to a set of
assignments to $\bar{x}^\mathsf{f}$ (i.e., a relation).
Following the notation in Ullman~\cite{ullman1985implementation}, 
a query $Q$ of arity $a(Q)$ is written as $Q^{\alpha}(\bar x)$ where $\alpha \in
\set{\mathsf{b},\mathsf{f}}^{a(Q)}$.
By default, all DMOs have the all-free binding pattern. But if a task exposes the 
access pattern of its DMOs, \felix can select evaluation strategies of the 
DMOs more informatively -- \felix employs a cost-based optimizer for DMOs
that takes advantage of both the RDBMS's cost-estimation facility and
the data-access pattern of a task (see Section~\ref{sec:optim}).

\begin{example}
Say the subprogram $F_1$-$F_5$ in Figure~\ref{fig:mlns} is executed
as a task that performs coreference resolution on $\rel{pCoref}$, and \felix chooses 
the correlation clustering algorithm~\cite{ailon2008aggregating,arasu2009large} for this task. 
At this point, \felix knows the data-access properties of that
algorithm (which essentially asks only for ``neighboring''
elements). \felix represents this using the following adorned view:
\begin{eqnarray*}
\rel{DMO}^\mathsf{bf}(x,y) &\la& \rel{affil}(x,o), \rel{affil}(y,o), \rel{pSimSoft}(x,y).
\end{eqnarray*}
\noindent
which is adorned as $\mathsf{bf}$. During execution, this
coref task sends requests such as $x=\,$\con{Joe}, and expects to receive a
set of names $\{y\mid\rel{DMO}($\con{Joe}$,y)\}$.
\end{example}

\noindent
Sometimes \felix could deduce from the DMOs how a task may be
parallelized (e.g., via key attributes), and takes advantage
of such opportunities.
\eat{A sufficient condition for
such parallelization opportunities is when the heads of all DMOs of a task share
a bound attribute that partitions the work done by the task. Call such an attribute
a \emph{pivot}. A task may inform \felix of the existence of pivots,
so that \felix can decide on degree of parallelism for the task based on
available resources (e.g., CPU cores and RAM).
}
The output of optimization is a DAG of {\em statements}. 
Statements are of two forms: (1) a prepared
SQL statement; (2) a statement encoding the necessary information to
run a task (e.g., the number of iterations an
algorithm should run, data locations,
etc.).

\subsection{Execution}
\label{sec:arch:execution}

In \felix, a process called the \emph{Master} coordinates the tasks by 
periodically updating the Lagrangian
multiplier associated with each shared tuple (e.g., $\lambda_p$ in
Example~\ref{ex:happysad}). Such an iterative updating scheme
is called \emph{master-slave message passing}. The goal is to optimize 
$\max_\lambda g(\lambda)$ using standard subgradient 
methods~\cite[p.~174]{wolsey1998integer}.
Specifically, let $p$ be an unknown tuple of $R$,
then at step $k$ the Master updates each $\lambda^i_p$ s.t. $R\in\mathcal{R}_i$ 
using the following rule:
\[
\lambda^i_p = \lambda^i_p + \alpha_k
\left(
x^i_p - \frac{\sum_{j:R\in\mathcal{R}_j} x^j_p }{|\{j:R\in\mathcal{R}_j\}|}
\right),
\]
where $\alpha_k$ is the gradient step size for this update.
A key novelty of \felix is that we can leverage the underlying RDBMS
to efficiently compute the gradient on an entire relation.
To see why, let $\lambda^j_p$ be the multipliers for a shared tuple $p$ of a relation $R$;
$\lambda^j_p$ is stored as an extra attribute in each copy $j$ of $R$.
Note that at each iteration, $\lambda^j_p$ changes only if the copies of $R$ do
not agree on $p$ (e.g., exactly one copy has $p$ missing).
Thus, we can update all $\lambda^j_p$'s with an outer join between the copies
of $R$ using SQL.
The gradient descent procedure stops either when all copies have reached
an agreement (or only a very small portion disagrees) or when \felix has
run a pre-specified maximum number of iterations.

\paragraph*{Scheduling and Parallelism}
Between two iterations of message passing, each task is executed until completion.
If these tasks run sequentially (say due to limited RAM or CPU), then any order of execution
would result in the same run time. On the other hand, if all tasks can run in parallel,
then faster tasks would have to wait for the slowest task to finish until message passing
could proceed. To better utilize CPU time, \felix updates the Lagrangian multipliers for
a shared relation $R$ whenever all involved tasks have finished. Furthermore, a task is
restarted when all shared relations of this task have been updated.
If computation resources are abundant, \felix also considers
parallelizing a task.

\paragraph*{Initialization and Finalization}
Let $\sigma=T_1,\ldots,T_n$ be a sequence of all tasks obtained by a breadth-first traversal of the
logical plan. At initial execution time, to bootstrap from the initial empty state, we sequentially execute
the tasks in the order of $\sigma$, each task initializing its local copies of a relation by copying
from the output of previous tasks. Then \felix performs the above master-slave message-passing
scheme for several iterations; during this phase all tasks could run in parallel.
At the end of execution, we perform a finalization step: 
we traverse $\sigma$ again and output the copy from $T^R_{\text{last}}$ 
for each query relation $R$, where $T^R_{\text{last}}$ is the last task in $\sigma$
that outputs $R$.
To ensure that hard rules in the input \mln program are not violated in the final output,
we insist that for any query relation $R$, $T^R_{\text{last}}$ respects all hard rules involving $R$.
(We allow hard rules to be assigned to multiple tasks.)
This guarantees that the output of the finalization step
is a possible world for $\Gamma$ (provided that the hard rules are satisfiable).

\section{Technical Details}
\label{sec:tech}
Having set up the general framework, in this section, we discuss further 
technical challenges and solutions in \felix. 
First, as each individual task might be as complex as the original \mln,
decomposition by itself does not automatically lead to high scalability.
To address this issue, we identify several common statistical tasks with
well-studied algorithms and characterize their correspondence with \mln
subprograms (Section~\ref{sec:specoper}).
Second, even when each individual task is able to run efficiently, sometimes
the data movement cost may be prohibitive. To address
this issue, we propose a novel cost-based materialization strategy for
data movement operators (Section~\ref{sec:optim}).
Third, since the user may not be able to provide a good task decomposition scheme,
it is important for \felix to be able to compile an \mln program into tasks
automatically. To support this, we describe the compiler of \felix 
that automatically recognizes specialized tasks in an \mln program 
(Section~\ref{sec:compilation}).

\subsection{Specialized Tasks} 
\label{sec:specoper}

\begin{table}
\centering
{\small
\begin{tabular}{ll}
  \hline\noalign{\vspace{1pt}}
  \textbf{Task} & \textbf{Implementation} \\
  \hline\noalign{\vspace{1pt}}
  Simple Classification & Linear models~\cite{boyd:cvx}\\
  Correlated Classification & Conditional Random Fields~\cite{lafferty2001conditional}\\
  Coreference & Correlation clustering~\cite{ailon2008aggregating,arasu2009large}\\
 \eat{ Generic \mln inference & \tuffy~\cite{tuffy-vldb11} \\}
  \hline
\end{tabular}
}
\caption{Example specialized tasks and their implementations in \felix.}
\label{tab:ops}
\end{table}

By default, \felix solves a task (which is also an \mln program) with
a generic \mln inference algorithm based on a reduction to MaxSAT~\cite{kautz1997general},
which is designed to solve sophisticated \mln programs. Ideally, when a
task has certain properties indicating that it can be solved using a
more efficient specialized algorithm, \felix should do so. 
Conceptually, the \felix framework supports all statistical tasks 
that can be modeled as mathematical programs.
As an initial proof of concept, our prototype of \felix integrates
two statistical tasks that are widely used in text applications:
classification and coreference (see Table~\ref{tab:ops}). These
specialized tasks are well-studied and so have algorithms with high
efficiency and high quality.
\\

\noindent
\textbf{Classification}
Classification tasks are ubiquitous in text applications; e.g.,
classifying documents by topics or sentiments, and classifying noun
phrases by entity types.  In a classification task, we are given a set
of objects and a set of labels; the goal is to assign a label to
each object. Depending on the structure of the cost function, there
are two types of classification tasks: \emph{simple classification}
and \emph{correlated classification}.

In simple classification, given a model, the assignment of each object
to a label is independent from other object labels. We describe
a Boolean classification task for simplicity, i.e., our goal is to
determine whether each object is in or out of a single
class. The input to a Boolean classification task is a pair of
relations: {\em the model} which can be viewed as a relation
$M(\underline{f},w)$ that maps each feature $f$ to a single
weight $w \in \R$, and a relation of objects $I(o,f)$; if a tuple
$(o,f)$ is in $I$ then object $o$ has feature $f$ (otherwise not). The
output is a relation $R(o)$ that indicates which objects are members of the
class ($R$ can also contain their marginal probabilities). For simple
classification, the optimal $R$ can be populated by including those
objects $o$ such that
\[ \sum_{ w : M(w,f) \text{ and } I(o,f) } w \geq 0 \]
One can implement a simple classification task with SQL aggregates,
which should be much more efficient than the MaxSAT algorithm used in
generic \mln inference.

The twist in \felix is that the objects and the features of the model
are defined by \mln rules.  For example, the rules $F_6$ and $F_7$ in
Figure~\ref{fig:mlns} form a classification task that determines
whether each $\rel{affil}$ tuple (considered as an object) holds. Said
another way, each rule is a feature. So, \felix populates the model
relation $M$ with two tuples: $M(F_6,+\infty)$ and $M(F_7,8)$, and
populates the input relation $I$ by executing the conjunctive queries
in $F_6$ and $F_7$; e.g., from $F_7$ \felix generates tuples of the
form $I(P,O,F_7)$, which indicates that the object
$\rel{affil}(P,O)$ has the feature $F_7$.\footnote{\scriptsize In general
a model usually has both positive and negative features.}
Operationally \felix performs such translation via DMOs that are also
adorned with the task's access patterns; e.g., the DMO for $I$ has the
adornment $I^{\mathsf{bbf}}$ since \felix classifies each
$\rel{affil}(P,O)$ independently.

\felix extends this basic model in two ways: (1) \felix implements multi-class classification
by adding a \emph{class} attribute to $M$ and $I$. (2) \felix also
supports \emph{correlated classification}: in addition to
per-object features, \felix also allows features that span multiple objects. For
example, in named entity recognition if we see the token {\it``Mr.''}
the next token is very likely to be a person's name. In general, one
can form a graph where the nodes are objects and two objects are
connected if there is a rule that refers to both objects. When this
graph is acyclic, the task essentially consists of tree-structured CRF
models that can be solved in polynomial time with dynamic programming
algorithms~\cite{lafferty2001conditional}.
\\

\noindent
\textbf{Coreference} Another common task is coreference
resolution (coref), e.g., given a set of strings (say phrases in a
document) we want to decide which strings represent the same
real-world entity.  These tasks are ubiquitous in text processing. The
input to a coref task is a single relation $B(\underline{o1,o2}, wgt)$
where $wgt=\beta_{o1,o2}\in\R$ indicates how likely the objects
$o1,o2$ are coreferent (with 0 being neutral). The output of a coref
task is a relation $R(o1,o2)$ that indicates which pairs of objects
are coreferent -- $R$ is an equivalence relation, i.e., satisfying
reflexivity, symmetry, and transitivity.  Assuming that
$\beta_{o1,o2}=0$ if $(o1,o2)$ is not in the key set of the relation
$B$, then each valid $R$ incurs a cost (called \emph{disagreement
cost})
\[
\costcoref(R) = \sum_{\substack{o1,o2: (o1,o2)\notin R \\ \mbox{ and } \beta_{o1,o2}>0}} |\beta_{o1,o2}|
+ \sum_{\substack{o1,o2: (o1,o2)\in R \\ \mbox{ and } \beta_{o1,o2}<0}} |\beta_{o1,o2}|.
\]
The goal of coref is to find a relation with the minimum cost:
\[
R^* = \arg\min_{R} \costcoref(R).
\]

\noindent
Coreference resolution is a well-studied
problem~\cite{fellegi1969theory,arasu2009large}. The
underlying inference problem is $\mathsf{NP}$-hard in almost all
variants. As a result, there is a literature on approximation
techniques (e.g., {\em correlation clustering}~\cite{ailon2008aggregating,arasu2009large}).
\felix implements these algorithms for coreference tasks.
In Figure~\ref{fig:mlns}, $F_1$ through $F_5$ consist of a coref task
for the relation $\rel{pCoref}$. $F_1$ through $F_3$ encode the
reflexivity, symmetry, and transitivity properties of $\rel{pCoref}$,
and $\rel{F_4}$ and $\rel{F_5}$ essentially define the weights on
the edges (similar to Arasu~\cite{arasu2009large}) from which \felix
constructs the relation $B$ (via DMOs).


\subsection{Optimizing Data Movement Operators}
\label{sec:tradeoff}
\label{sec:optim}

Recall that data are passed between tasks and the RDBMS
via data movement operators (DMOs). 
While the statistical algorithm inside a task
may be very efficient (Section~\ref{sec:specoper}), DMO evaluation
could be a major scalability bottleneck. An important goal of \felix's
optimization stage is to decide whether and how to materialize DMOs.
For example, a baseline approach would be to materialize all DMOs. 
While this is a reasonable approach
when a task repeatedly queries a DMO with the same parameters, 
in some cases, the result may be so large that an eager
materialization strategy would exhaust available disk space. For
example, on an Enron dataset, materializing the following DMO would
require over 1TB of disk space:

\vspace{-8pt}
{\small
\begin{eqnarray*}
\rel{DMO}^\mathsf{bb}(x,y) &\la& \rel{mention}(x,name1),\ \rel{mention}(y,name2), \\ \vspace{-3mm}
 && \rel{mayref}(name1,z),\ \rel{mayref}(name2,z).
\end{eqnarray*}
}
\vspace{-8pt}

\noindent
Moreover, some
specialized tasks may inspect only a small fraction of their
search space and so such eager materialization is inefficient. For
example, one implementation of the coref task is a stochastic
algorithm that examines data items roughly linear in the number of nodes (even
though the input to coref contains a quadratic number of pairs of
nodes)~\cite{arasu2009large}.
In such cases, it seems more reasonable to simply declare the DMO as a regular database 
view (or prepared statement) that is to be evaluated lazily during execution.

\felix is, however, not confined to fully eager or fully lazy. In
\felix, we have found that intermediate points (e.g., materializing a
subquery of a DMO $Q$) can have dramatic speed 
improvements (see Section~\ref{sec:cost:material}).
To choose among materialization strategies, \felix takes hints from
the tasks:
\felix allows a task to expose its access patterns,
including both an adornment $Q^\alpha$ (see Section~\ref{sec:arch:optimization}) 
and an estimated number
of accesses $t$ on $Q$. (Operationally $t$ could be
a Java function or SQL query to be evaluated against the base relations of $Q$.)
Those parameters together with the cost-estimation facility of
the underlying RDBMS (here, PostgreSQL) enable a System-R-style cost-based
optimizer of \felix that explores all possible materialization strategies
using the following cost model.

\paragraph*{Felix Cost Model}
To define our cost model, we introduce some notation. Let $Q^\alpha(\bar
x)\la g_1,g_2,\ldots,g_k$ be a DMO.
Let $G=\{g_i|1\le i\le k\}$ be the set of
subgoals of $Q$.  Let $\mathcal{G}=\{G_1,\ldots,G_m\}$ be a partition of $G$;
i.e., $G_j\subseteq G$, $G_i\cap G_j=\emptyset$ for all $i\ne j$, and
$\bigcup G_j = G$.
Intuitively, a partition represents a
possible materialization strategy: each element of the partition
represents a query (or simply a relation) that \felix is considering
materializing. That is, the case of one $G_i=G$ corresponds to a fully
eager strategy. The case where all $G_i$ are singleton sets corresponds
to a lazy strategy.

More precisely, define $Q_j(\bar x_j) \la G_j$ where $\bar x_j$ is the
set of variables in $G_j$ shared with ${\bar x}$ or any other $G_i$
for $i\neq j$. Then, we can implement the DMO with a regular
database view $Q'(\bar x)\la Q_1,\ldots,Q_m$. Let $t$
be the total number of accesses on $Q'$ performed
by the statistical task. We model the
execution cost of a materialization strategy as:
\[
\textrm{ExecCost}(Q',t) = t\cdot \textrm{Inc}_{\alpha}(Q') + \sum_{i=1}^m \textrm{Mat}(Q_i)
\]
\noindent
$\textrm{Mat}(Q_i)$ is the cost of eagerly materializing $Q_i$
and $\textrm{Inc}_{\alpha}(Q')$ is the estimated cost of each
query to $Q'$ with adornment $\alpha$.

A significant
implementation detail is that since the subgoals in $Q'$ are not
actually materialized, we cannot directly ask PostgreSQL for the
incremental cost $\textrm{Inc}_{\alpha}(Q')$.\footnote{{\scriptsize
    PostgreSQL does not fully support ``what-if'' queries, although
    other RDBMSs do, e.g., for indexing tuning.}}  In our prototype
version of \felix, we implement a simple approximation of PostgreSQL's
optimizer (that assumes incremental plans use only index-nested-loop
joins), and so our results should be taken as a lower bound on the
performance gains that are possible when materializing one or more subqueries.
\rev{We provide more details on this approximation in Section~\ref{app:cost-estimation}.} \edit{add}
Although the number of possible plans is exponential in the
size of the largest rule in an input Markov Logic program, in our
applications the individual rules are small. Thus, we can estimate
the cost of each alternative, and we pick the one with the lowest
$\mathrm{ExecCost}$.

\subsection{Automatic Compilation}
\label{sec:compilation}
\begin{table}
\centering \small
\begin{tabular}{lll}
  \hline\noalign{\vspace{1pt}}
  \textbf{Properties} & \textbf{Symbol} & \textbf{Example  }\\
  \hline\noalign{\vspace{1pt}}
  Reflexive & REF & $p(x,y) \implies p(x,x)$\\
  Symmetric  & SYM & $p(x,y) \implies p(y,x)$  \\
  Transitive  & TRN  & $p(x,y), p(y,z) \implies p(x,z)$  \\
  Key & KEY & $p(x,y),p(x,z)\implies y=z$  \\
  \noalign{\vspace{1pt}}\hline\noalign{\vspace{1pt}}
  Not Recursive           & NoREC & Can be defined w/o Recursion.\\
  Tree Recursive & TrREC &   See Equation~\ref{eq:po}\\
\hline
\end{tabular}
\caption{Properties assigned to predicates by the \felix
  compiler. KEY refers to
  a non-trivial key. Recursive properties are derived from all rules;
  the other properties are derived from hard rules.}
\label{tab:rule-properties}
\end{table}

\begin{table}
\centering
\small
\begin{tabular}{ll}
  \hline\noalign{\vspace{1pt}}
  \textbf{Task} & \textbf{Required Properties} \\
  \hline\noalign{\vspace{1pt}}
  Simple Classification  & KEY, NoREC\\
  Correlated Classification & KEY, TrREC\\
  Coref & REF, SYM, TRN\\
  Generic MLN Inference & none\\
  \hline
\end{tabular}
\caption{Tasks and their required properties.}
\label{tab:op-rule}
\end{table}
So far we have assumed that the mappings between \mln rules, tasks, and
algorithms
are all specified by the user. However, ideally a compiler should be able
to automatically recognize
subprograms that could be processed as specialized tasks.
In this section we describe \rev{a best-effort compiler that is able to
automatically detect the presence of classification and coref tasks}
\edit{was ``such a compiler''}.
To decompose an \mln program $\Gamma$ into tasks, \felix uses a two-step
approach. \felix's first step is to annotate each query predicate $p$ with a
set of {\em properties}. An example property is whether or not $p$ is
symmetric.  Table~\ref{tab:rule-properties} lists of the set of
properties that \felix attempts to discover with their
definitions; NoREC and TrREC are rule-specific.
Once the properties are
found, \felix uses Table~\ref{tab:op-rule} to list all possible
options for a predicate.\eat{\footnote{\scriptsize The proofs of correctness of
these rules are straightforward.}}
When there are multiple options, the current prototype of \felix simply chooses
the first task to appear in the following order:
(Coref, Simple Classification, Correlated Classification, Generic). 
This order intuitively favors more specific tasks.
To compile an \mln into tasks, \felix greedily applies the above
procedure to split a subset of rules into a task, and then iterates
until all rules have been consumed.
As shown below,  property detection is non-trivial 
as the predicates are the output of
SQL queries (or formally, datalog programs).
Therefore, \felix implements a best-effort compiler using a
set of syntactic patterns; this compiler is sound but not complete.
It is interesting future work to design
more sophisticated compilers for \felix.

\paragraph*{Detecting Properties}
The most technically difficult part of the compiler is determining the
properties of the predicates (cf.~\cite{Fan:VLDB2008}). 
There are two types of properties that
\felix looks for: (1) schema-like properties of any possible worlds
that satisfy $\Gamma$ and (2) graphical structures of correlations
between tuples. For both types of properties, the challenge
is that we must infer these properties from the underlying rules
applied to an infinite number of databases.\footnote{\scriptsize As is
  standard in database theory~\cite{abiteboul1995foundations}, to
  model the fact the query compiler runs without examining the data,
  we consider the domain of the attributes to be unbounded. If the
  domain of each attribute is known then, all of the above properties
  are decidable by the trivial algorithm that enumerates all (finitely
  many) instances.}
For example, SYM is the property:
\begin{quote}
``{\it for any database $I$ that satisfies $\Gamma$, does the sentence
    $\forall x,y. \mathrm{pCoref}(x,y) \iff \mathrm{pCoref}(y,x)$
    hold?}''.
\end{quote}
Since $I$ comes from an infinite set, it is not immediately clear that
the property is even decidable. Indeed, $\text{REF}$ and $\text{SYM}$ 
are not decidable for Markov Logic
programs.

Although the set of properties in Table~\ref{tab:rule-properties} is
motivated by considerations from statistical inference, the first four properties
depend {\it only on the hard rules in $\Gamma$}, i.e., the constraints
and (SQL-like) data transformations in the program. Let
$\Gamma_{\infty}$ be the set of rules in $\Gamma$ that have infinite
weight. We consider the case when $\Gamma_{\infty}$ is written as a
datalog program.

\begin{theorem}
Given a datalog program $\Gamma_{\infty}$, a predicate $p$, and a
property $\theta$ deciding if for all input databases $p$ has property
$\theta$ is undecidable if $\theta \in
\set{\text{REF},\text{SYM}}$.
\end{theorem}

\eat{
\begin{proof}[Sketch]
If there is a single rule with $Q$ such as $Q(x,y) <= Q1(x),Q2(y)$
then $Q$ is $\set{\text{REF},\text{SYM}}$ if and only if $Q1$ or $Q2$
is empty or $Q1 \equiv Q2$. We assume that $Q1$ and $Q2$ are
satisfiable. If there is an instance where $Q1(a)$ and $Q2$ is false
for all values. Then there is another world (with all fresh constants)
where $Q2$ is true (and does not return $a$). Thus, to check
$\text{REF}$ and $\text{SYM}$ for $Q$, we decide need to decide
equivalence of datalog queries. Equivalence of datalog queries is
undecidable~\cite[ch.~12]{abiteboul1995foundations}. Since containment
and boundedness for monadic datalog queries is decidable, a small
technical wrinkle is that while $Q1$ and $Q2$ are arity one (monadic)
their bodies may contain other recursive (higher arity)
predicates.
\end{proof}
}

The above result is not surprising as datalog is a powerful language
and containment is undecidable~\cite[ch.~12]{abiteboul1995foundations}
(the proof reduces from containment). Moreover, the compiler
is related to {\em implication problems} studied by Abiteboul and
Hull (who also establish that generalizations of KEY
and TRN problem are
undecidable~\cite{DBLP:conf/sigmod/AbiteboulH88}). $\text{NoREC}$ is
the negation of the {\em boundedness
  problem}~\cite{gmsv93:surajits:paper} which is undecidable.

In many cases, recursion is not used in $\Gamma_{\infty}$ (e.g.,
$\Gamma_{\infty}$ may consist of standard SQL queries that transform
the data), and so a natural restriction is to consider
$\Gamma_{\infty}$ without recursion, i.e., as a union of conjunctive
queries.

\begin{theorem}
Given a union of conjunctive queries $\Gamma_{\infty}$, deciding if
for all input databases that satisfy $\Gamma_{\infty}$ the query
predicate $p$ has property $\theta$ where $\theta \in \set{\text{REF},
  \text{SYM}}$ (Table~\ref{tab:rule-properties}) is
decidable. Furthermore, the problem is $\Pi_{2}
\mathsf{P}$-Complete. KEY and TRN are trivially false. NoRec is
trivially true.
\label{thm:comp}
\end{theorem}

\eat{
\begin{proof}[Sketch]
The key observation is that the only symmetric or reflexive sets $S
\subseteq D \times D$ that can be written as $S = X \times Y$ must be
when either $X$, $Y$ is empty or if $X = Y$. Thus, $q(x,y) \mathrm{:-}
q_1(x), q_2(y)$ is reflexive or symmetric if and only if $q_1 \equiv
q_2$. Thus, we have reduced reflexive and symmetric to containment and
equivalence. Since our language allows us to express conjunctive
queries with inequality constraints, this established $\Pi_{2}
\mathsf{P}$ hardness~\cite{containment-jacm88}. Membership in $\Pi_{2}
\mathsf{P}$ is established using a standard technique: without
recursion our language has a {\em small model property}, which means
that we only need to examine instances whose size is a polynomial in
the number of free variables in the query.
\end{proof}
}

\eat{
This result is an example of one future direction: using database
constraints, e.g., key constraints and inclusion constraints, to
improve statistical inference algorithms.
}

Still, \felix must annotate predicates with properties. To cope with the
undecidability and intractability of finding out compiler annotations,
\felix uses a set of sound (but not complete) rules that are described
by simple patterns. For example, we can conclude that a predicate
$R$ is transitive if program contains syntactically the rule
$R(x,y),R(y,z) => R(x,z)$ with weight $\infty$. 
\eat{A complete list of
such expressions is in the full version of this paper.}\\

\noindent
\textit{Ground Structure} The second type of properties that \felix
considers characterize the graphical structure of the {\em ground
  database} (in turn, this structure describes the correlations that
must be accounted for in the inference process). We assume that
$\Gamma$ is written as a datalog program (with stratified negation). The
ground database is a function of both soft and hard rules in the input
program, and so we consider both types of rules here. \felix's
compiler attempts to deduce a special case of recursion that is
motivated by (tree-structured) conditional random fields that we call
TrREC. Suppose that there is a single recursive rule that contains $p$
in the body and the head is of the form:
\begin{equation}
p(x,y), T(\underline{y},z) => p(x,z)
\label{eq:po}
\end{equation}

\noindent
where the first attribute of $T$ is a key and the transitive closure
of $T$ is a partial order. In the ground database, $p$ will be
``tree-structured''. MAP and marginal inference for such rules are in
\textsf{P}-time~\cite{Wainwright:2008:GME:1523420, sen09-PrDB}. \felix
has a regular expression to deduce this property.

\eat{
We illustrate the compilation process by example.

\begin{example}
Consider the labeling example in Section 2.2, the relation
$\rel{label}(phrase,label)$ is labeled as KEY (from
$(\gamma_{2.1.1})$) and we get TrREC from deducing that $(d,p)$ and
$(d,p')$ are both candidate keys of $\rel{next}(d,p,p')$ (and there are
no other syntactically recursive rules for $\rel{next})$. So,
according to Table \ref{tab:op-rule}, $\rel{label}$ can be solved by
the Labeling task or the generic MLN inference algorithm. \felix
chooses the Labeling task as it is more specific than \tuffy.
\end{example}
}


\section{Experiments}
\label{sec:experiments}
Although \mln inference has a wide range of applications,
we focus on knowledge-base construction tasks.
In particular, we use \felix 
to implement the TAC-KBP
challenge;
\felix is able to scale to the 1.8M-document corpus and 
produce results with state-of-the-art quality.
In contrast, prior (monolithic) approaches to \mln inference crash
even on a subset of KBP that is orders of magnitude smaller.

In Section~\ref{sec:exp-overall}, we compare the overall scalability
and quality of \felix with prior \mln inference approaches
on four datasets (including KBP).
We show that, when prior \mln systems run, \felix is able to
produce similar results but more efficiently; when prior \mln
systems fail to scale, \felix can still generate high-quality results.
In Sections~\ref{sec:exp:lagrangian}, 
we demonstrate that the message-passing scheme in \felix can effectively reconcile
conflicting predictions and has stable convergence behaviors.
In Section~\ref{sec:exp:tasks}, 
we show that specialized tasks and algorithms are critical
for \felix's high performance and scalability.
In Section~\ref{sec:cost:material}, we validate that the cost-based DMO
optimization is crucial to \felix's efficiency.

\begin{table}[!t]\centering\small
  \begin{tabular}{lrr}
      \hline\noalign{\vspace{1pt}}
      & \textbf{\#documents} & \textbf{\#mentions} \\
      \noalign{\vspace{1pt}}\hline\noalign{\vspace{1pt}}
    \textbf{KBP} &  1.8M & 110M \\
    \textbf{Enron} & 225K & 2.5M  \\
    \textbf{DBLife} &  22K & 700K  \\
    \textbf{NFL} &  1.1K & 100K \\
    \eat{\textbf{Enron-R} & 680 & 486  \\}
    \hline
  \end{tabular}
  \caption{Statistics of input data. Note that \mln inference
generates much larger intermediate data.}\label{tab:data-stats}
\end{table}

\eat{
\begin{table}[thb]\centering\small
  \begin{tabular}{|l||r|r|r|}
      \hline
      & \textbf{\#bytes} & \textbf{\#documents} & \textbf{\#mentions} \\
      \hline
    \textbf{KBP} & 8.7GB & 1.8M & 110M \\
    \textbf{Enron} & 700MB & 225K & 2.5M  \\
    \textbf{DBLife} & 300MB & 22K & 700K  \\
    \textbf{NFL} & 20MB & 1.1K & 100K \\
    \eat{\textbf{Enron-R} & 1.6MB & 680 & 486  \\}
    \hline
  \end{tabular}
  \caption{Statistics of input data.}\label{tab:data-stats}
\end{table}
}

\paragraph*{Datasets and Applications}
Table~\ref{tab:data-stats} lists some statistics about the four
datasets that we use for experiments:
(1) \textbf{KBP} is a 1.8M-document corpus from TAC-KBP; the task
is to perform two related tasks: a) {\em entity linking}: extract all
entity mentions and map them to entries in Wikipedia, and b) {\em slot
filling}: determine (tens of types of) relationships between entities. There is also a set of
ground truths over a 2K-document subset (call it \textbf{KBP-R}) 
that we use for quality assessment.
(2) \textbf{NFL}, where the task is to extract football game
results (winners and losers) from sports news articles.
(3) \textbf{Enron}, where the task is to
identify person mentions and associated phone numbers in the Enron
email dataset. There are two versions of Enron:
{Enron}\footnote{
  \scriptsize{\url{http://bailando.sims.berkeley.edu/enron_email.html}}}
is the full dataset;
\textbf{Enron-R}\footnote{
  \scriptsize{\url{http://www.cs.cmu.edu/~einat/datasets.html}}} is a
680-email subset that we manually annotated person-phone ground
truth on.  We use Enron
for performance evaluation, and Enron-R for quality assessment.
(4) \textbf{DBLife}\footnote{
  \scriptsize{\url{http://dblife.cs.wisc.edu}}}, where the task is to
extract persons, organizations, and affiliation relationships between
them from a collection of academic webpages. For {DBLife},
we use the ACM author profile 
data\eat{\footnote{\scriptsize{\url{http://www.acm.org/membership/author_pages}}}}
as ground truth.

\paragraph*{MLN Programs}
For {KBP}, we developed \mln programs that fuse a wide array of
data sources including NLP results, Web search results, Wikipedia
links, Freebase, etc.  For performance experiments, we use our entity
linking program (which is more sophisticated than slot filling).
The \mln program on {NFL} has a conditional random field model
as a component, with some additional common-sense rules (e.g., ``a
team cannot be both a winner and a loser on the same day.'') that are
provided by another research project.  To expand our set of \mln
programs, we also create
\mlns on {Enron} and {DBLife} by adapting rules in 
state-of-the-art rule-based IE approaches~\cite{liu2010automatic,dblife2007}:
Each rule-based program is essentially equivalent to an MLN-based
program (without weights).  We simply replace the ad-hoc reasoning in
these deterministic rules by a simple statistical variant. For example, the
DBLife program in \cimple~\cite{dblife2007} says that if a person and an organization
co-occur with some regular expression context then they are
affiliated, and ranks relationships by frequency of such
co-occurrences.  In the corresponding \mln we have several rules for
several types of co-occurrences, and
ranking is by marginal probabilities.

\paragraph*{Experimental Setup}
To compare with alternate implementations of \mlns, we consider
two state-of-the-art \mln implementations: (1) \alc, the
reference implementation for \mlns~\cite{alchemy:website}, and (2) \tuffy,
an RDBMS-based implementation of \mlns~\cite{tuffy-vldb11}.
\alc is implemented in C++. \tuffy and \felix are both implemented in
Java and use PostgreSQL 9.0.4. \felix uses \tuffy as a task. Unless
otherwise specified, all experiments are run on a RHEL5 workstation
with two 2.67GHz Intel Xeon CPUs (24 total cores), 24 GB of RAM, and
over 200GB of free disk space.

\subsection{High-level Scalability and Quality}
\label{sec:exp-overall}

\begin{table}[!thb]\centering\small
  \begin{tabular}{lccc}
      \hline\noalign{\vspace{1pt}}
       Scales? &  \textbf{\felix} & \textbf{\tuffy} & \textbf{\alc}\\
      \noalign{\vspace{1pt}}\hline\noalign{\vspace{1pt}}
    \textbf{KBP}  &  Y & N &  N\\
    \textbf{NFL}  &  Y & Y &  N\\
    \textbf{Enron}  &  Y &  N & N \\
    \textbf{DBLife}  &  Y &  N & N\\
    \textbf{KBP-R}  &  Y & N &  N\\
    \textbf{Enron-R}  &  Y & Y &  N\\
    \hline
  \end{tabular}
  \caption{Scalability of various \mln systems.}
  \label{tab:mln:scale}
\end{table}

We empirically validate that \felix achieves higher scalability and
essentially identical result quality compared to
prior monolithic approaches.
To support these claims, we compare the performance and
quality of different MLN inference systems (\tuffy, \alc, and \felix)
on the datasets listed above: KBP, Enron, DBLife, and NFL.
In all cases, \felix runs its automatic compiler; parameters
(e.g., gradient step sizes, generic inference parameters)
are held constants across datasets.
\tuffy and \alc have two sequential phases in their run time: 
\emph{grounding} and \emph{search}; results are produced
only in the search phase.  A system is deemed unscalable if it fails
to produce any inference results within 6 hours.  The overall
scalability results are shown in Table~\ref{tab:mln:scale}.

\begin{figure*}[!t]
  \centering
  \includegraphics[width=1\textwidth]{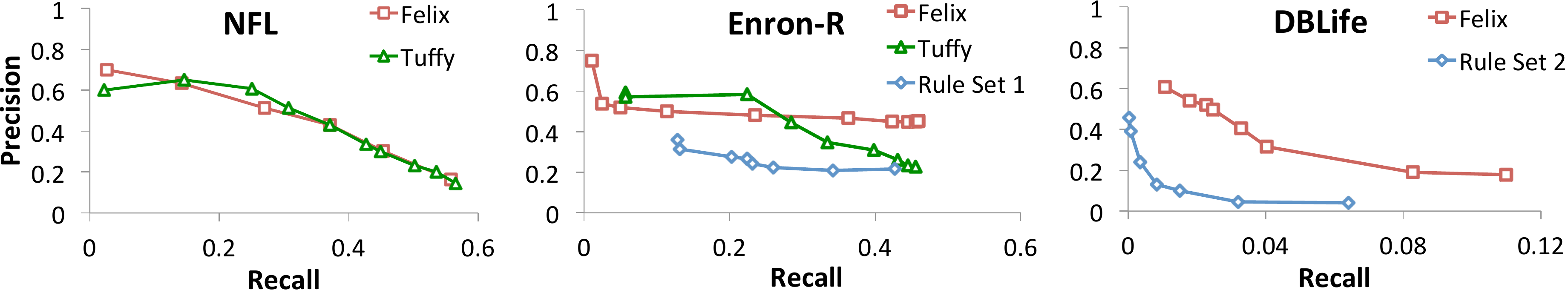} \caption{High-level
  quality results of various MLN systems. For each dataset, we plot a precision-recall
curve of each system by varying $k$ in top-k results; missing curves indicate that
a system does not scale on the corresponding dataset. }
  \label{fig:pr-all3}
\end{figure*}

\paragraph*{Quality Assessment}
We perform quality assessment on four datasets: KBP-R, NFL, Enron-R,
and DBLife.  On each dataset, we run each \mln system for 4000 seconds
with marginal inference. (After 4000 seconds, the quality of each
system has stabilized.)  For KBP-R, we convert the output to TAC's
query-answer format and compute the F1 score against the ground truth.
For the other three datasets, we draw precision-recall curves: we take
ranked lists of predictions from each system and measure
precision/recall of the top-k results while varying the number of
answers returned\footnote{Results from \mln-based systems are ranked
by marginal probabilities, results from \cimple are ranked by
frequency of occurrences, and results from rules on Enron-R are
ranked by window sizes between a person mention and a phone
number mention.}. 
The quality of each system is shown in
Figure~\ref{fig:pr-all3}\footnote{{\scriptsize The low recall on
DBLife is because the ground truth (ACM author profiles) contains many
facts absent from DBLife.}}.  System-dataset pairs that do not scale
have no curves.

\paragraph*{KBP \& NFL}
Recall that there are two tasks in KBP: entity linking and slot filling.
On both tasks, \felix is able to scale to the 1.8M documents and
after running about 5 hours on a 30-node parallel RDBMS, produce
results with state-of-the-art quality\rev{\cite{ji2010overview}}\edit{add}\footnote{\scriptsize Measured on KBP-R that
has ground truth.}: 
We achieved an F1 score 0.80 on entity linking
(human annotators' performance is 0.90), and an F1 score 0.34 on slot filling
(state-of-the-art  quality). In contrast, \tuffy and \alc
crashed even on the three orders of magnitude smaller {KBP-R} 
subset. Although also based on an RDBMS, \tuffy attempted
to generate about $10^{11}$ and $10^{14}$ tuples on KBP-R
and KBP, respectively.

To assess the quality of \felix as compared to monolithic inference,
we also run the three \mln systems on NFL.  Both \felix and \tuffy
scale on the NFL data set, and as shown in Figure~\ref{fig:pr-all3},
produce results with similar quality.  However, \felix is an order of
magnitude faster: \tuffy took about an hour to start outputting
results, whereas \felix's quality converges after only five
minutes. We validated that the reason is that \tuffy was not aware of
the linear correlation structure of a classification task in the NFL
program, and ran generic
\mln inference in an inefficient manner.

\paragraph*{Enron \& DBLife}
\rev{
To expand our test cases, we consider two more datasets -- 
Enron-R and DBLife -- to evaluate the key question we try to answer:
\emph{does \felix outperform monolithic
systems in terms of scalability and efficiency}?
From Table~\ref{tab:mln:scale}, we
see that \felix scales in cases where monolithic \mln systems do not.
On Enron-R (which contains only 680 emails), we see that when
both \felix and \tuffy scale, they achieve similar result
quality. From Figure~\ref{fig:pr-all3}, we see that even when
monolithic systems fail to scale (on DBLife), \felix is able to
produce high-quality results.}

\rev{
To understand the result quality obtained by \felix, 
we also ran rule-based information-extraction 
programs for Enron-R and DBLife following practice described in the 
literature~\cite{michelakis2009uncertainty,liu2010automatic,dblife2007}.
Recall that the \mln programs for Enron-R and DBLife
were created by augmenting the deterministic rule sets 
with statistical reasoning.}\footnote{
For Enron-R, we followed the rules described in related 
publications~\cite{michelakis2009uncertainty,liu2010automatic} .
For DBLife, we obtained the \cimple~\cite{dblife2007} system and
the DBLife dataset from the authors.
Further details can be found in Section~\ref{app:exp-quality}.
}\rev{
It should be noted that all systems can be improved with further tuning.
In particular,  the rules described in the 
literature (``Rule Set 1'' for Enron-R~\cite{michelakis2009uncertainty,liu2010automatic}
and ``Rule Set 2'' for DBLife~\cite{dblife2007}) were
not specifically optimized for high quality on the corresponding tasks. 
On the other hand, the corresponding \mln programs were generated 
in a constrained manner (as described in Section~\ref{app:exp-quality}). 
In particular, we did not leverage state-of-the-art NLP tools nor 
refine the \mln programs. With these caveats in mind,
from Figure~\ref{fig:pr-all3} we see that
(1) on Enron-R, \felix achieves higher precision
than Rule Set 1 given the same recall; 
and (2) on DBLife, \felix
achieves higher recall than Rule Set 2
(i.e., \cimple~\cite{dblife2007}) at any precision level.
This provides preliminary indication that 
statistical reasoning could help improve
the result quality of knowledge-base construction tasks, 
and that scaling up \mln inference is a promising
approach to high-quality knowledge-base construction.
Nevertheless, it is interesting future work to more deeply
investigate \emph{how} statistical reasoning contributes
to quality improvement over deterministic rules (e.g., 
Michelakis et al.~\cite{michelakis2009uncertainty}).
}

\eat{
Since \mln inference combines logical rules and statistical
reasoning, it is interesting to see whether the addition of statistical reasoning
enables \mln-based approaches to knowledge-base construction
to improve upon pure rule-based approaches.
An ideal comparison would involve heavily-tuned or industry-strength systems from 
both sides, but that would require months of engineering effort.
Absent such resources, we narrow the scope and compare the quality
of \felix versus the rule-based IE systems (namely \syst and \cimple) with rules
that were used to produce the \mln programs for Enron and DBLife.\footnote{We downloaded
\syst from \url{http://www.alphaworks.ibm.com/tech/systemt} 
(which however now redirects to another page as of Feb 2012),
and then closely followed the rules described in publications on 
\syst~\cite{michelakis2009uncertainty,liu2010automatic} 
We obtained the \cimple~\cite{dblife2007} system and
the DBLife dataset from the authors.
Further details can be found in Section~\ref{app:exp-quality}.
} Such comparisons are limited by the basic rule sets that are 
available to us, and so the quality numbers do not
represent the full potential of either \mln-based
or rule-based approaches to KBC.
Nevertheless, from Figure~\ref{fig:pr-all3} we see that the quality of \mln 
systems are significantly higher than rule-based systems,
which suggests that scaling up \mln inference could be a promising
approach to high-quality knowledge-base construction.
A more systematic comparison is interesting future work.
}

\subsection{Effectiveness of  Message Passing}
\label{sec:exp:lagrangian}

\begin{figure}[!t]
  \centering
  \includegraphics[width=0.48\textwidth]{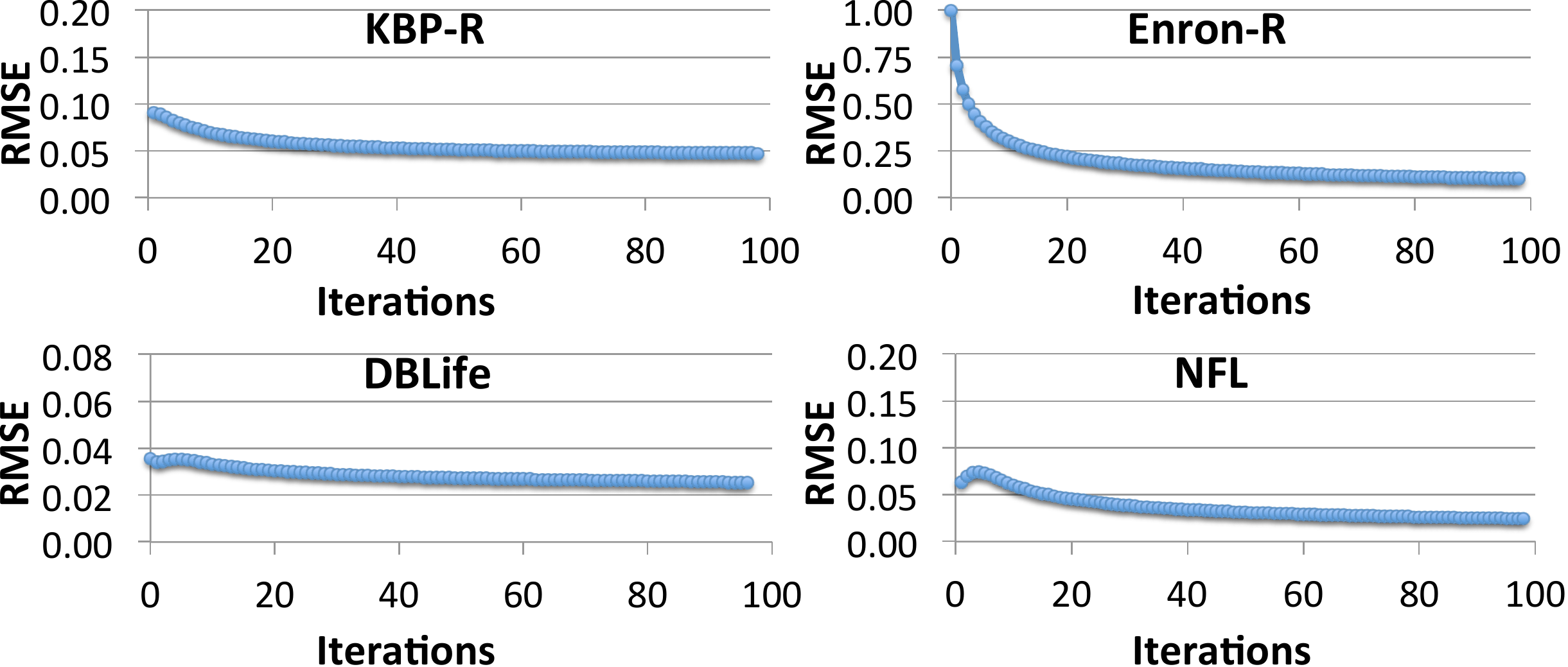} 
  \caption{The RMSE between predictions from different tasks converges stably
as \felix runs master-slave message passing.}
  \label{fig:convergence}
\end{figure}

We validate that the Lagrangian scheme in \felix can effectively
reconcile conflicting predictions between related tasks to produce
consistent output.  Recall that \felix uses master-slave message
passing to iteratively reconcile inconsistencies between different
copies of a shared relation. To validate that this scheme is
effective, we measure the difference between the marginal
probabilities reported by different copies; we plot this difference
as \felix runs 100 iterations.  Specifically, we measure the
root-mean-square-deviation (RMSE) between the marginal predictions of
shared tuples between tasks.  On each of the four datasets (i.e.,
KBP-R, Enron-R, DBLife, and NFL), we plot how the RMSE changes over
time. As shown in Figure~\ref{fig:convergence},
\felix stably reduces the RMSE on all datasets to an eventual value of below
$0.1$ -- after about 80 iterations on Enron and after the very first
iteration for the other three datasets. (As many statistical inference algorithms
are stochastic, it is expected that the RMSE does not decrease to
zero.)  This demonstrates
that \felix can effectively reconcile conflicting predictions, thereby
achieving joint inference.

\mln inference is NP-hard, 
and so it is not always the case that \felix converges to the exact
optimal solution of the original program.  However, as we validated in
the previous section, empirically \felix converges to close
approximations of monolithic inference results (only more
efficiently).

\eat{
\begin{figure*}[!t]
  \centering
  \includegraphics[width=1\textwidth]{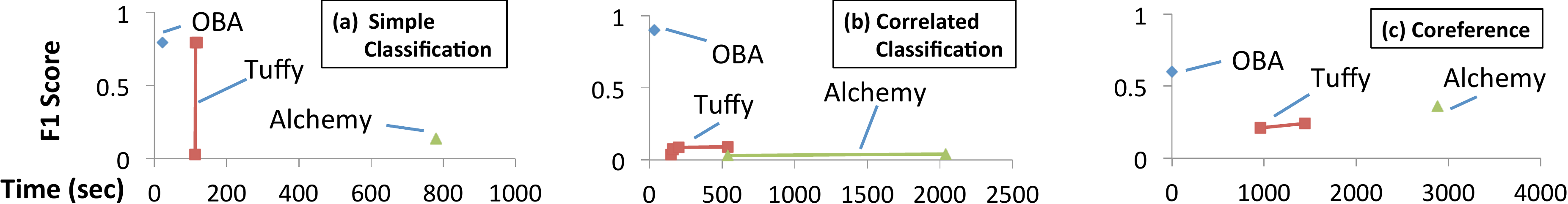} \caption{Comparison
  of \felix, \tuffy, and \alc on specialized tasks.}
  \label{fig:spec}
\end{figure*}
}

\subsection{Importance of Specialized Tasks}
\label{sec:exp:tasks}
We validate that the ability to integrate specialized tasks into \mln
inference is key to \felix's higher performance and scalability.  To
do this, we first show that specialized algorithms
have higher efficiency than generic \mln inference on individual
tasks. Second, we validate that specialized tasks are key to \felix's
scalability on \mln inference.

\begin{table}[!t]\centering\small
  \begin{tabular}{p{2cm}lrrr}
      \hline\noalign{\vspace{1pt}}
      \textbf{Task} & \textbf{System} & \textbf{Initial} & \textbf{Final} & \textbf{F1}\\
      \noalign{\vspace{1pt}}\hline\noalign{\vspace{1pt}}
    \multirow{3}{*}{\parbox{2cm}{\textbf{Simple\\Classification}}}    
	& \felix & 22 sec & 22 sec & 0.79 \\
      	& \tuffy & 113 sec & 115 sec & 0.79 \\
       	& \alc & 780 sec & 782 sec & 0.14 \\ 
      \noalign{\vspace{1pt}}\hline\noalign{\vspace{1pt}}
    \multirow{3}{*}{\parbox{2cm}{\textbf{Correlated\\Classification}}}  
    	& \felix & 34 sec & 34 sec & 0.90 \\
    	& \tuffy & 150 sec & 200 sec & 0.09 \\
    	& \alc & 540 sec & 560 sec & 0.04 \\ 
      \noalign{\vspace{1pt}}\hline\noalign{\vspace{1pt}}
    \multirow{3}{*}{\parbox{2cm}{\textbf{Coreference}}}  
    	& \felix & 3 sec & 3 sec & 0.60 \\ 
    	& \tuffy & 960 sec & 1430 sec & 0.24 \\
    	& \alc & 2870 sec & 2890 sec & 0.36 \\
    \hline
  \end{tabular}
  \caption{Performance and quality comparison on individual tasks.
``Initial'' (resp. ``Final'') is the time when a system produced
the first (resp. converged) result. ``F1'' is the F1 score of the 
final output.
}
  \label{tab:tasks}
\end{table}

\paragraph*{Quality \& Efficiency}
We first demonstrate that \felix's specialized algorithms outperform
generic MLN inference algorithms in both quality and performance when solving
specialized tasks. To evaluate this claim, we run \felix, \tuffy, and \alc on
three MLN programs that each encode one of the following
tasks: simple classification, correlated classification, and coreference. 
We use a subset of the Cora dataset\footnote{{\scriptsize \url{http://alchemy.cs.washington.edu/data/cora}}}
for coref, and a subset of the CoNLL 2000 chunking dataset\footnote{{\scriptsize \url{http://www.cnts.ua.ac.be/conll2000/chunking/}}} for
classification.
The results are shown in Table~\ref{tab:tasks}.
While it always takes less than a minute for \felix to
finish each task, \tuffy and \alc take much
longer.\eat{The reason is that \tuffy and \alc are not aware of 
the structure of a specialized task, and resort to generic
inference procedures that compute and store
the search space very inefficiently.}
Moreover, the quality of \felix is higher than \tuffy and
\alc. As expected,  \felix can achieve exact
optimal solutions for classification, and nearly optimal approximation
for coref, whereas \tuffy and \alc rely on a general-purpose SAT
counting algorithm. Nevertheless, the
above micro benchmark results are typically drowned out in
larger-scale applications, where the quality difference tend to be
smaller compared to the results here.

\paragraph*{Scalability}

To demonstrate that specialized tasks are crucial to the scalability of \felix, we
remove specialized tasks from \felix and re-evaluate whether \felix
is still able to scale to the four datasets (KBP, Enron, DBLife, and NFL).
The results are as follows: after disabling classification, 
\felix crashes on KBP and DBLife; after disabling coref, \felix crashes on Enron.
On NFL, although \felix is still able to run without specialized tasks,
its performance slows down by an order of magnitude (from less than five minutes
to more than one hour).
These results suggest that specialized tasks are critical to \felix's
high scalability and performance.

\subsection{Importance of DMO Optimization}
\label{sec:cost:material}
We validate that \felix's cost-based approach to data movement optimization
is crucial to the efficiency of \felix.  
To do this, we run \felix on subsets of Enron with various sizes in three different
settings: 1) \textbf{Eager},
where all DMOs are evaluated eagerly; 2)
\textbf{Lazy}, where all DMOs are evaluated lazily;
3) \textbf{Opt}, where \felix decides the materialization strategy for
each DMO based on the cost model in
Section~\ref{sec:tradeoff}.

\begin{table}[thb]\centering\small
  \begin{tabular}{lrrrr}
      \hline\noalign{\vspace{1pt}}
       & \textbf{E-5k} & \textbf{E-20k} & \textbf{E-50k} & \textbf{E-100k}\\
      \hline\noalign{\vspace{1pt}}
    \textbf{Eager}   & 83 sec & 15 min & 134 min & 641 min \\
    \textbf{Lazy}    &  42 sec & 5 min & 22 min & 78 min\\
    \textbf{Opt}   & 29 sec & 2 min & 7 min & 25 min\\
     \hline
  \end{tabular}
  \caption{DMO efficiency under different settings.}
  \label{tab:view-mat}
\end{table}

We observed that overall \textbf{Opt} is substantially more efficient
than both \textbf{Lazy} and \textbf{Eager}, and found that the
deciding factor is the efficiency of the DMOs of the coref tasks.
Thus, we specifically measure the total run time of individual coref
tasks, and compare the results in Table~\ref{tab:view-mat}.
Here, \textbf{E-$x$k} for $x \in \set{5,20,50,100}$ refers to a
randomly selected subset of $x$k emails in the Enron corpus. We
observe that the performance of the eager materialization strategy
degrades rapidly as the dataset size increases. The lazy strategy
performs much better. The cost-based approach can further achieve
2-3X speedup. This demonstrates that our cost-based materialization
strategy for data movement operators is crucial to the efficiency
of \felix.

\section{Conclusion and Future Work}
\label{sec:conclusion}

We present our \felix approach to \mln inference that uses
relation-level Lagrangian relaxation to decompose an \mln program into
multiple tasks and solve them jointly.  Such task decomposition
enables \felix to integrate specialized algorithms for common tasks
(such as classification and coreference) with both high efficiency and
high quality.  To ensure that tasks can communicate and access data
efficiently, \felix uses a cost-based materialization strategy for
data movement.  To free the user from manual task decomposition, the
compiler of \felix performs static analysis to find specialized tasks
automatically.  Using these techniques, we demonstrate that \felix is
able to scale to complex knowledge-base construction applications and
produce high-quality results whereas previous \mln systems \rev{have much
poorer scalability} \edit{was ``crash'' which was perhaps too categorical}.  Our
future work is in two directions: First, we plan to apply our key
techniques (in-database Lagrangian relaxation and cost-based
materialization) to other inference problems. Second, we plan to
extend \felix with new logical tasks and physical implementations
to support broader applications.\\[1pt]

\bibliographystyle{abbrv}
\bibliography{felix}

\pagebreak
\appendix

\newcommand{\vtheta}{\boldsymbol{\theta}}
\newcommand{\vx}{\boldsymbol{x}}
\newcommand{\vy}{\boldsymbol{y}}
\newcommand{\vphi}{\boldsymbol{\phi}}
\newcommand{\vmu}{\boldsymbol{\mu}}
\newcommand{\vnu}{\boldsymbol{\nu}}
\newcommand{\vw}{\boldsymbol{w}}
\newcommand{\vxi}{\boldsymbol{\xi}}
\newcommand{\xset}{\mathcal{X}}

\section{Notations}
Table~\ref{tab:notations} defines some common notation that is used in
the following sections.

\begin{table}[thb]\centering\small
  \begin{tabular}{ccc}
  \hline
  Notation       &    Definition \\
  \hline
  $a,b,\ldots,\alpha,\beta,\ldots$         & Singular (random) variables \\
  $\boldsymbol{a}$, $\boldsymbol{b}$,$\ldots$, $\boldsymbol{\alpha}$, $\boldsymbol{\beta}$,$\ldots$   & Vectorial (random) variables\\
  $\vmu' \cdot \vnu$  & Dot product between vectors\\
  $|\vmu|$       & Length of a vector or size of a set\\
  $\vmu_i$      & $i^{th}$ element of a vector\\
  $\hat{\alpha},\hat{\boldsymbol{\alpha}}$      & A value of a variable\\
  \hline
  \end{tabular}
  \caption{Notations}
  \label{tab:notations}
\end{table}

\section{Theoretical Background of the Operator-based Approach}

In this section, we discuss the theoretical underpinning of \felix's
operator-based approach to \mln inference.  Recall that \felix first
decomposes an input \mln program based on a predefined set of
operators, instantiates those operators with code selection, and then executes
the operators using ideas from dual decomposition.  We first justify
our choice of specialized subtasks (i.e., Classification, Sequential
Labeling, and Coref) in terms of two compilation soundness and
language expressivity properties:

\begin{enumerate}
\item Given an \mln program, the subprograms obtained by \felix's
  compiler indeed encode specialized subtasks such as classification,
  sequential labeling, and coref.

\item \mln as a language is expressive enough to encode all possible
  models in the exponential family of each subtask type; specifically,
  \mln subsumes logistic regression (for classification), conditional
  random fields (for labeling), and correlation clustering (for
  coref).
\end{enumerate}

We then describe how dual decomposition is used to coordinate the
operators in \felix for both MAP and marginal inference while
maintaining the semantics of \mlns.

\subsection{Consistent Semantics}

\subsubsection{MLN Program Solved as Subtasks}

In this section, we show that the decomposition of an \mln program produced by
\felix's compiler indeed corresponds to the subtasks defined
in Section~\ref{sec:tasks}.

\paragraph*{Simple Classification}
Suppose a classification operator (i.e., task) for a query relation $R(k,v)$ consists of
key-constraint hard rules together with rules $r_1,...,r_t$ (with weights $w_1,...,w_t$)
\footnote{For simplicity, we assume that these $t$ rules are ground formulas. It is easy
to show that grounding does not change the property of rules.}.
As per \felix's compilation procedure, the following holds:
1) $R(k,v)$ has a key constraint (say $k$ is the key); and
2) none of the selected rules are recursive with respect to $R$.

Let $k_0$ be a fixed value of $k$. Since $k$ is a possible-world key for $R(k,v)$,
we can partition the set of all possible worlds into sets based on
their $v$ for $R(k_0,v)$ (and whether there is any value $v$ make $R(k,v)$ true).
Let $\mathcal{W}_{v_i} = \{W \mid W \models R(k_0, v_i)\}$ and
$\mathcal{W}_\perp$ where $R(k_0,v)$ is false for all
$v$. Define $Z(\mathcal{W}) = \sum_{w \in \mathcal{W}} \exp \{
-cost(w)\}$.  Then according to the semantics of MLN,

\[ \Pr[R(k,v_0)] = \frac{Z(\mathcal{W}_{v_0})}{Z(\mathcal{W}_{\perp}) + \sum_{v \in \mathbb{D}}
Z(\mathcal{W}_v)} \]

It is immediate from this that each class is disjoint. It is also
clear that, conditioned on the values of the rule bodies, each of the $R$
are independent.

\paragraph*{Correlated Classification}
Suppose a correlated classification operator outputs a relation $R(k,v)$ and consists
of hard-constraint rules together with ground rules $r_1,...r_t$ (with
weights $w_1,...,w_t$). As per \felix's compilation procedure, the
following holds:
\begin{itemize}
\item $R(k,v)$ has a key constraint (say $k$ is the key);

\item The rules $r_i$ satisfy the TrREC property.

\end{itemize}

Consider the following graph: the nodes are all possible values for
the key $k$ and there is an edge $(k,k')$ if $k$ appears in the body
of $k'$. Every node in this graph has outdegree at most $1$. Now
suppose there is a cycle: But this contradicts the definition of a
strict partial order. In turn, this means that this graph is a
forest. Then, we identify this graph with a graphical model structure
where each node is a random variable with domain $\mathbb{D}$. This is
a tree-structured Markov random field. This justifies the rules used by
\felix's compiler for identifying labeling operators. Again,
conditioned on the rule bodies any grounding is a tree-shaped
graphical model.

\paragraph*{Coreference Resolution}

A coreference resolution subtask involving variables $y_1,...y_n$
infers about an equivalent relation $R(y_i,y_j)$. The only
requirement of this subtask is that the result relation $R(.,.)$
be reflexive, symmetric and transitive.
\felix ensures these properties by detecting corresponding
hard rules directly.

\subsubsection{Subtasks Represented as MLN programs}

We start by showing that all probabilistic distributions in
the discrete exponential
family can be represented by an equivalent MLN program.
Therefore, if we model the three subtasks using models in the exponential
family, we can express them as an MLN program. Fortunately,
for each of these subtasks, there are popular exponential family
models: 1) Logistic Regression (LR) for Classification,
2) Conditional Random Filed (CRF) for Labeling and
3) Correlation Clustering for Coref. \footnote{We leave the discussion
of models that are not explicitly in exponential family to
future work.}

\begin{definition}[Exponential Family]
We follow the definition in \cite{Wainwright:2008:GME:1523420}.
Given a vector of binary random variables $\vx \in \mathcal{X}$,
let $\vphi: \mathcal{X} \rightarrow \{0,1\}^{d}$ be
a binary vector-valued function. For a given $\vphi$,
let $\vtheta \in \mathbb{R}^d$ be a vector of real number parameters.
The exponential family distribution over $\vx$ associated with $\vphi$
and $\vtheta$ is of the form:

\[
\Pr_{\vtheta}[\vx]=\exp\{-\vtheta \cdot \vphi(\vx)-A(\vtheta)\},
\]

\noindent
where $A(\vtheta)$ is known as log partition function:
$A(\vtheta)=\log \sum_{\vx \in \mathcal{X}} \exp\{-\vtheta \cdot \vphi(\vx)\}$.
\end{definition}

This definition extends to multinomial random variables
in a straightforward manner.
For simplicity, we only consider binary random variables
in this section.

\begin{example}
Consider a textbook logistic regressor over a random variable $x \in \{0,1\}$:

\[
\Pr[x=1] = \frac{1}{1 + \exp\{\sum_i -\beta_if_i\}},
\]

\noindent
where $f_i \in \{0,1\}$'s are known as features of $x$ and $\beta_i$'s
are {\em regression coefficients} of $f_i$'s. This
distribution is actually in the exponential family: Let $\vphi$ be a
binary vector-valued function whose $i^{th}$ entry equals
to $\vphi_i(x) = (1-x)f_i$. Let $\vtheta$ be a vector of
real numbers whose $i^{th}$ entry $\vtheta_i = \beta_i$.
One can check that
\[
\begin{aligned}
\Pr[x=1] = & \frac{\exp \left\{-\vtheta \cdot \vphi(1)\right\}}
            {\exp\left\{-\vtheta \cdot \vphi(1) \right\}
             + \exp\left\{-\vtheta \cdot \vphi(0) \right\}}\\
         = & \frac{1}{1 + \exp\set{\sum_i -\beta_if_i}}
\end{aligned}
\]
\end{example}

The exponential family has a strong connection with
the maximum entropy principle and graphic models. For
all the three tasks we are considering, i.e.,
classification, labeling and coreference, there are popular
exponential family models for each of them.

\begin{proposition}
\label{thm:ef-mln}
Given an exponential family distribution over $\vx \in \mathcal{X}$
associated with $\vphi$ and $\vtheta$, there exists an MLN program
$\Gamma$ that defines the same probability distribution as
$\Pr_{\vtheta}[\vx]$.  The length of the formula in $\Gamma$ is at
most linear in $|\vx|$, and the number of formulas in $\Gamma$ is at
most exponential in $|\vx|$.
\end{proposition}

\begin{proof}
Our proof is by construction. Each entry of $\vphi$ is a binary
function $\phi_i(\vx)$, which partitions $\mathcal{X}$ into two
subsets: $\mathcal{X}_i^+ = \{\vx | \phi_i(\vx)=1\}$ and
$\mathcal{X}_i^- = \{\vx | \phi_i(\vx)=0\}$. If $\theta_i \ge 0$, for
each $\hat{\vx} \in \mathcal{X}_i^+$, introduce a rule:
\[ \theta_i ~~~ \bigvee_{1\le j \le |\vx|} R(x_j, 1-\hat{x_j}). \]
If $\theta_i < 0$, for each $\hat{\vx} \in \mathcal{X}_i^+$,
insert a rule:
\[ -\theta_i ~~~ \bigwedge_{1\le j \le |\vx|} R(x_j, \hat{x_j}). \]
We add these rules for each $\phi_i(.)$, and also add
the following hard rule for each variable $x_i$:
\[ \infty ~~~ R(x_i, 0)\quad <=>\quad \neg R(x_i, 1).\]
It is not difficult to see $\Pr[\forall x_i, R(x_i,\hat{x_i}) = 1] = \Pr_{\vtheta}[\hat{\vx}]$.
In this construction, each formula has length
$|\vx|$ and there are $\sum_i (|\mathcal{X}_i| + 1)$
formulas in total, which is exponential in $|\vx|$ in the worst case.
\end{proof}

Similar constructions apply to the case where $\vx$ is a vector of multinomial
random variables.

We then show that Logistic Regression, Conditional Random Field
and Correlation Clustering all define probability distributions
in the discrete exponential family, and the number of formulas
in their equivalent MLN program $\Gamma$ is polynomial in the number
of random variables.

\paragraph*{Logistic Regression}

In Logistic Regression, we model the probability distribution
of Bernoulli variable $y$ conditioned on $x_1,...,x_k \in \{0,1\}$ by

\[ \Pr[y=1] = \frac{1}{1 + \exp\{-(\beta_0 + \sum_i \beta_i x_i)\}}  \]

Define $\phi_i(y) = (1-y)x_i$ ($\phi_0(y) = 1-y$) and $\theta_i = \beta_i$,
we can see $\Pr[y=1]$ is in the exponential family defined as
in Definition \ref{thm:ef-mln}. For each $\phi_i(y)$, there is only
one $y$ that can get positive value from $\phi_i$, so there are at
most $k+1$ formulas in the equivalent MLN program.

\paragraph*{Conditional Random Field}

In Conditional Random Field, we model the probability distribution
using a graph $G=(V,E)$ where $V$ represents the set of random variables
$\vy=\{y_v: v \in V\}$. Conditioned on a set of random variables $\vx$,
CRF defines the distribution:

\begin{equation*}
\begin{aligned}
\Pr[\vy|\vx] \propto ~& \exp \{ \sum_{v\in V, k} \lambda_k f_k (v, y_v, \vx) \\
 ~& + \sum_{(v_1,v_2) \in E, l } \mu_l g_l((v_1,v_2), y_{v_1}, y_{v_2}, \vx) \}
\end{aligned}
\end{equation*}

This is already in the form of exponential family. Because
each function $f_k(v,-,\vx)$ or $g_l((v_1,v_2),-,-,\vx)$ only relies
on 1 or 2 random variables, the resulting MLN program has at most
$O(|E|+|V|)$ formulas. In the current prototype of \felix, we only consider linear chain CRFs,
where $|E| = O(|V|)$.

\paragraph*{Correlation Clustering}

Correlation clustering is a form of clustering for which there are
efficient algorithms that have been shown to scale to instances of the
coref problem with millions of mentions. Formally, correlation
clustering treats the coref problem as a graph partitioning
problem. The input is a weighted undirected graph $G=(V,f)$ where $V$
is the set of mentions with weight function $f:V^2 \rightarrow
\mathbb{R}$. The goal is to find a partition $\mathcal{C}=\{C_i\}$ of
$V$ that minimizes the \emph{disagreement cost}:

\[cost_{cc}(\mathcal{C}) =
\sum_{\substack{(v1,v2)\in V^2\\v_1 \ne v_2\\\exists C_i, v_1\in C_i \wedge v_2\in C_i\\f(u,v)<0}} |f(v_1,v_2)| +
\sum_{\substack{(v1,v2)\in V^2\\v_1 \ne v_2\\\exists C_i, v_1\in C_i \wedge v_2 \not \in C_i\\f(u,v)>0}} |f(v_1,v_2)|
 \]

We can define the probability distribution over $\mathcal{C}$ similarly
as MLN:

\[ \Pr[\mathcal{C}] \propto \exp\{-cost_{cc}(\mathcal{C})\}  \]

Specifically, let the binary predicate $coref(v_1,v_2)$ indicate whether
$v_1\ne v_2\in V$ belong to the same cluster. First introduce three
hard rules enforcing the reflexivity, symmetry, and transitivity properties
of $coref$. Next, for each $v_1\ne v_2\in V$, introduce a singleton
rule $coref(v_1,v_2)$ with weight $f(v_1,v_2)$.
It's not hard to show that the above distribution holds for this \mln program.

\subsection{Dual Decomposition for MAP and Marginal Inference}

In this section, we formally describe the dual decomposition framework
used in \felix to coordinate the operators. We
start by formalizing \mln inference as an optimization problem.
Then we show how to apply dual decomposition on these optimization problems.

\subsubsection{Problem Formulation}

Suppose an MLN program $\Gamma$ consists of a set of ground MLN rules
$\mathcal{R}=\{r_1,...,r_m\}$ with weights ($w_1,...,w_m$).
Let ${X}=\{x_1,...,x_n\}$ be the set of boolean random variables
corresponding to the ground atoms occurring in $\Gamma$.
Each MLN rule $r_i$ introduces a function $\phi_i$ over the
set of random variables $\pi_i\subseteq{X}$ mentioned in $r_i$:
$\phi_i(\pi_i)=1$ if $r_i$ is violated and $0$ otherwise.
Let $\vw$ be a vector of weights.
Define vector $\vphi({X}) =
(\phi_1(\pi_1),...,\phi_m(\pi_m))$.
Given a possible world $\vx\in2^{X}$, the cost can be represented:

\[ cost(\vx) =  \vw\cdot\vphi(\vx) \]

Suppose \felix decides to solve $\Gamma$
with $t$ operators $O_1,...,O_t$. Each operator $O_i$ contains a set of
rules $\mathcal{R}_{i}\subseteq\mathcal{R}$. The set $\{\mathcal{R}_{i}\}$
forms a partition of $\mathcal{R}$.
Let the set of random variables
for each operator be ${X}_{i} = \cup_{r_j \in
  \mathcal{R}_{i}} \pi_j$.
Let $n_i = |{X}_{i}|$.
Thus, each operator $O_i$ essentially solves the MLN program defined
by random variables ${X}_{i}$
and rules $\mathcal{R}_{i}$.
Given $\vw$, define $\vw^{i}$ to be the weight vector whose
entries equal $\vw$ if the corresponding rule appears in
$\mathcal{R}_{i}$ and 0 otherwise. Because
$\mathcal{R}_{i}$ forms a partition of $\mathcal{R}$,
we know $\sum_i \vw^{i} = \vw$.
For each operator $O_i$, define an $n$-dim vector $\vmu_{i}({X})$, whose
$j^{th}$ entry equals $x_j$ if $x_j \in {X}_{i}$
and 0 otherwise. Define $n$-dim vector $\vmu({X})$ whose $j^{th}$ entry equals
$x_j$.
Similarly, let $\vphi({X}_{i})$ be the projection of
$\vphi({X})$ onto the rules in operator $O_i$.

\begin{example}
We use the two sets of rules for classification and labeling in
Section \ref{sec:specoper} as a running example.  For a simple
sentence {\em Packers win.} in a fixed document $D$ which contains two phrases
$P_1 =$ ``Packers'' and $P_2 =$ ``win'', we will get the following
set of ground formulae \footnote{For $r_{l1}$, $p \in \{P_1,P_2\}, l_i
  \in \{W,L\}$.}:

\begin{center}
{\small
\begin{tabular}{@{\hspace{-10pt}}ll@{\hspace{1pt}}l}
$\infty$ & $\rel{label}(D, p, l1), \rel{label}(D, p, l2) => l1=l2$ & $(r_{l1})$\\
10 & $\rel{next}(D,P_1,P_2), \rel{token}(P_2,\text{`wins'})  => \rel{label}(D, P_1, W)$ & $(r_{l2})$\\
1  & $\rel{label}(D, P_1, W), \rel{next}(D, P_1, P_2) => !\rel{label}(D, P_2, W)$ & $(r_{l3})$\\

10 & $\rel{label}(D, P_1, W), \rel{referTo}(P_1, GreenBay) => \rel{winner}(GreenBay)$ & $(r_{c1})$\\
10 & $\rel{label}(D, P_1, L), \rel{referTo}(P_1, GreenBay) => \rel{!winner}(GreenBay)$ & $(r_{c2})$\\
\end{tabular}
}
\end{center}

\noindent
After compilation, \felix would assign $r_{l1}$, $r_{l2}$ and $r_{l3}$ to a labeling operator $O_L$,
and $r_{c1}$ and $r_{c2}$ to a classification operator $O_C$. For each of
$\{\rel{winner}(GreenBay)$, $\rel{label}(D, P_1, W)$, $\rel{label}(D, P_1, L)$,
$\rel{label}(D, P_2, W)$, $\rel{label}(D, P_2, L)\}$ we have a binary random variable
associated with it. Each rule introduces a function $\phi$, for example,
the function $\phi_{l2}$ introduced by $r_{l2}$ is:

\[
\phi_{l2}(\rel{label}(D, P_1, W)) =
\begin{cases}
1 & \text{if $\rel{label}(D, P_1, W)=\textrm{False}$}\\
0 & \text{if $\rel{label}(D, P_1, W)=\textrm{True}$}
\end{cases}
\]

The labeling operator $O_L$ essentially solves the
MLN program with variables ${X}_{L} = \{\rel{label}(D, P_1, W)$, $\rel{label}(D, P_1, L)$,
$\rel{label}(D, P_2, W)$, $\rel{label}(D, P_2, L)\}$ and
rules $\mathcal{R}_{L} = \{r_{l1}$, $r_{l2}$, $r_{l3}\}$.
Similarly $O_C$ solves the MLN program with variables
${X}_{C} = \{\rel{winner}(GreenBay)$, $\rel{label}(D, P_1, W)$
$\rel{label}(D, P_1, L)\}$ and rules $\mathcal{R}_{C} = \{r_{c1}$, $r_{c2}\}$.
Note that these two operators share the variables
$\rel{label}(D, P_1, W)$ and $\rel{label}(D, P_1, L)$.
\label{example:packer}
\end{example}

\subsubsection{MAP Inference}

MAP inference in MLNs is to find an assignment ${\vx}$ to $X$ that
minimizes the cost:

\begin{equation}
\label{map-goal}
\min_{{\vx} \in \{0,1\}^n}  \vw\cdot\vphi({\vx}).
\end{equation}

Each operator $O_i$ performs MAP inference on ${X}_{i}$:

\begin{equation}
\label{map-op-goal}
\min_{{\vx}_{i} \in \{0,1\}^{n_i}}  \vw^{i}\cdot\vphi({\vx}_{i}).
\end{equation}

Our goal is to reduce the problem represented by Eqn.~\ref{map-goal} into
subproblems represented by Eqn.~\ref{map-op-goal}.
Eqn.~\ref{map-goal} can be rewritten as

\[ \min_{\vx \in \{0,1\}^n}  \sum_{1\le i\le t} \vw^{i} \cdot\vphi(\vx_{i}). \]

Clearly, the difficulty lies in that, for $i\ne j$,
${X}_{i}$ and ${X}_{j}$ may overlap.
Therefore, we introduce a copy of variables for each $O_i$:
${X}_{i}^C$. Eqn.~\ref{map-goal} now becomes:

\begin{equation}
\begin{aligned}
\min_{{\vx}_{i}^C \in \{0,1\}^{n_i},{\vx}}  & ~~ \sum_{i} \vw^{i}
\cdot\vphi(\vx_{i}^C)\\
s.t.                                               & ~~
\forall i \quad\vx_i^C = \vx.
\end{aligned}
\end{equation}

The Lagrangian of this problem is:

\begin{equation}
\begin{aligned}
& \mathcal{L}(\vx,\vx_{1}^C,...,\vx_{t}^C,\vnu_1,...,.\vnu_t) \\
= &\sum_{i} \vw^{i}\cdot \vphi(\vx_{i}^C) + \vnu_i \cdot
(\vmu_{i}(\vx_{i}^C) - \vmu_{i}(\vx))
\end{aligned}
\end{equation}

Thus, we can relax Eqn.~\ref{map-goal} into

\[
\max_{\vnu} \left\{ \sum_{i} \left[ \min_{\vx_{i} \in \{0,1\}^{n_i}}
    \vw^{i}\cdot \vphi(\vx_{i}^C) + \vnu_i \cdot
\vmu_{i}(\vx_{i}^C) \right]  -  \max_{\vx} \sum_{i}  \vnu_i \cdot \vmu_{i}(\vx)\right\}
\]

The term
$\max_{\vx} \sum_{i}  \vnu_i \cdot \vmu_i(\vx)=\infty$ unless
for each variable $x_j$,

\[ \sum_{O_i: x_j \in {X}_{i}} \vnu_{i,j} = 0.\]

Converting this into constraints, we get

\begin{equation*}
\begin{aligned}
\max_{\vnu}  &~~ \left\{ \sum_{i} \min_{\vx_{i} \in \{0,1\}^{n_i}}
    \vw^{i}\cdot \vphi(\vx_{i}^C) + \vnu_i \cdot
\vmu_{i}(\vx_{i}^C) \right\}\\
s.t.                 &~~ \forall x_j\quad\sum_{O_i: x_j \in {X}_{i}}\vnu_{i,j} = 0
\end{aligned}
\end{equation*}

We can apply sub-gradient methods on $\vnu$. The dual decomposition
procedure in \felix works as follows:

\begin{enumerate}
\item Initialize $\vnu_1^{(0)},...,\vnu_t^{(0)}$.
\item At step $k$ (starting from 0):
\begin{enumerate}
\item For each operator $O_i$, solve the MLN program
consisting of: 1) original rules in this operator, which
are characterized by $\vw^{i}$; 2) additional priors on
each variables in ${X}_{i}$, which are characterized
by $\vnu_i^{(k)}$.
\item Get the MAP inference results $\hat{\vx_{i}^C}$.
\end{enumerate}
\item Update $\vnu_i$: \\$\vnu_{i,j}^{(k+1)} = \vnu_{i,j}^{(k)}
  - \lambda \left( \hat{\vx_{i,j}^C} - \frac{\sum_{l: x_j \in
    X_{l}} \hat{\vx_{l,j}^C}}{|\{l:x_j \in
  X_{l}\}|} \right)$
\end{enumerate}

\begin{example}
Consider the MAP inference on program in Example \ref{example:packer}.
As $O_L$ and $O_C$ share two random variables: $x_w = \rel{label}(D, P_1, W)$ and
$x_l = \rel{label}(D, P_1, L)$, we have a copy of them for each
operator: $x_{w,O_L}^C$, $x_{l,O_L}^C$ for $O_L$;
and $x_{w,O_C}^C$, $x_{l,O_C}^C$ for $O_C$.
Therefore, we have four $\nu$: $\nu_{w,O_L}$, $\nu_{l,O_L}$
for $O_L$; and $\nu_{w,O_C}$, $\nu_{l,O_C}$
for $O_C$. Assume we initialize each $\nu_-^{(0)}$ to $0$
at the first step.

We start by performing MAP inference on $O_L$ and $O_C$ respectively.
In this case, $O_L$ will get the result:

\[
\begin{aligned}
x_{w,O_L}^C = 1\\
x_{l,O_L}^C = 0
\end{aligned}
\]

\noindent
$O_C$ admits multiple possible worlds minimizing the cost;
for example, it may outputs

\[
\begin{aligned}
x_{w,O_C}^C = 0\\
x_{l,O_C}^C = 0
\end{aligned}
\]

\noindent
which has cost 0. Assume the step size $\lambda = 0.5$. We can
update $\nu$ to:

\[
\begin{aligned}
\nu_{w,O_L}^{(1)} = & -0.25\\
\nu_{w,O_C}^{(1)} = & 0.25\\
~~ &\\
\nu_{l,O_L}^{(1)} = & 0\\
\nu_{l,O_C}^{(1)} = & 0\\
\end{aligned}
\]

Therefore, when we use these $\nu_-^{(1)}$ to conduct MAP inference
on $O_L$ and $O_C$, we are equivalently adding

\begin{center}
{\small
\begin{tabular}{@{\hspace{-10pt}}ll@{\hspace{1pt}}l}
-0.25 $\rel{label}(D, P_1, W)$ & $(r_{l}')$\\
\end{tabular}
}
\end{center}

\noindent
into $O_L$ and

\begin{center}
{\small
\centering
\begin{tabular}{@{\hspace{-10pt}}ll@{\hspace{1pt}}l}
0.25 $\rel{label}(D, P_1, W)$ & $(r_{c}')$\\
\end{tabular}
}
\end{center}

\noindent
into $O_C$. Intuitively, one may interpret this procedure
as the information that ``$O_L$ prefers $\rel{label}(D, P_1, W)$ to be true''
being passed to $O_C$ via $r_{c}'$.

\end{example}

\subsubsection{Marginal Inference}

The marginal inference of MLNs aims at computing the marginal
distribution (i.e., the expectation since we are dealing with
boolean random variables):

\begin{equation}
\label{marginal-goal}
\hat{\vmu} = \mathbb{E}_{\vw}[\vmu({X})].
\end{equation}

The sub-problem of each operator is of the form:

\begin{equation}
\label{marginal-op-goal}
\hat{\vmu}_{O} = \mathbb{E}_{\vw_O}[\vmu_O({X}_O)].
\end{equation}

Again, the goal is to use solutions for
Eqn.~\ref{marginal-op-goal} to solve Eqn.~\ref{marginal-goal}.

We first introduce some auxiliary variables. Recall that
$\vmu({X})$ corresponds to the set of random variables,
and $\vphi({X})$ corresponds to all functions represented
by the rules. We create a new vector
$\vxi$ by concatenating $\vmu$ and $\vphi$: $\vxi({X}) =
(\vmu^T({X}), \vphi^T({X}))$. We create a new weight
vector $\vtheta = (0,...,0,\vw^T)$ which is of the same length
as $\vxi$. It is not difficult to see that the marginal inference problem
equivalently becomes:

\begin{equation}
\label{marginal-goal2}
\hat{\vxi} = \mathbb{E}_{\vtheta}[\vxi({X})].
\end{equation}

Similarly, we define $\vtheta_O$ for operator $O$ as
$\vtheta_O = (0,...,0,\vw^T_O)$. We also define
a set of $\vtheta$: $\Theta_O$, which contains all
vectors with entries corresponding to random variables
or cliques not appear in operator $O$ as zero.
The partition function $A(\vtheta)$ is:

\[ A(\vtheta) = \sum_{\xset} \exp\{-\vtheta\cdot \vxi(\xset)\} \]

The conjugate dual to $A$ is:

\[ A^*(\vxi) = \sup_{\vtheta} \{  \vtheta\cdot \vxi - A(\vtheta) \} \]

A classic result of variational inference~\cite{Wainwright:2008:GME:1523420}
 shows that

\begin{equation}
\label{variational-goal}
\hat{\vxi} = \arg \sup_{\vxi \in \mathcal{M}} \{  \vtheta\cdot \vxi - A^*(\vxi) \},
\end{equation}

\noindent
where $\mathcal{M}$ is the marginal polytope. Recall that $\hat{\vxi}$ is our
goal (see Eqn.~\ref{marginal-goal2}). Similar to MAP inference, we want
to decompose Eqn.~\ref{variational-goal} into different operators by
introducing copies of shared variables. We first try to decompose
$A^*(\vxi)$. In $A^*(\vxi)$, we search $\vtheta$ on all possible
values for $\vtheta$. If we only search on a subset of $\vtheta$, we can
get a lower bound:

\[ A^{*O}(\vxi) = \sup_{\vtheta \in \Theta_O} \{  \vtheta\cdot \vxi -
A^*(\vxi) \} \leq A^*(\vxi).\]

\noindent
Therefore,

\[ - A^*(\vxi) \leq \frac{1}{m}\sum_O -A^{*O}(\vxi), \]
where $m$ is the number of operators.
We approximate $\hat{\vxi}$ using this bound:

\[ \hat{\vxi} = \arg \sup_{\vxi \in \mathcal{M}} \{  \vtheta\cdot \vxi
- \frac{1}{m}\sum_O A^{*O}(\vxi)\}, \]

\noindent
which is an upper bound of the original goal. We introduce
copies of $\vxi$:

\begin{equation*}
\begin{aligned}
\hat{\vxi} = \arg \sup_{\vxi^{O_i} \in \mathcal{M}, \vxi} &~~ \{  \sum_O \vtheta_O \cdot \vxi^O
- \frac{1}{m}\sum_O A^{*O}(\vxi^O)\}\\
                    s.t.                                      &~~
                    \vxi^O_e = \vxi_e, \forall e \in \mathcal{X}_O \cup
                    \mathcal{R}_O, \forall O
\end{aligned}
\end{equation*}

The Lagrangian of this problem is:

\[
\begin{aligned}
\mathcal{L}(\vxi,\vxi^{O_1},...,\vxi^{O_t},\vnu_1,...,\vnu_t) &=
\sum_O \left\{ \vtheta_O \cdot \vxi^O -\frac{1}{m} A^{*O}(\vxi^O) \right\} \\
&~~+\sum_{i} \vnu_i\cdot (\vxi^{O_i} - \vxi),
\end{aligned}
\]

\noindent
where $\vnu_i \in \Theta_{i}$, which means only the entries
corresponding to random variables or cliques that appear in
operator $O_i$ are allowed to have non-zero values. We get
the relaxation:

\[
\begin{aligned}
\min_{\vnu_i \in \Theta_{i}}
\sum_{i} \sup_{\vxi^{O_i} \in \mathcal{M}} \left\{ \vtheta_{i} \cdot \vxi^{O_i} 
- \frac{1}{m}A^{*O_i}(\vxi^{O_i}) +
\vnu_i \cdot \vxi^{O_i} \right\} \\- \min_{\vxi} \sum_{i} \vnu_i \cdot
\vxi
\end{aligned}
\]

Considering the $\min_{\vxi} \sum_{i} \vnu_i \cdot \vxi$ part. This
part is equivalent to a set of constraints:

\begin{equation*}
\begin{aligned}
\sum_{O_i: x \in {X}_{i}}
\vnu_{i,x} =& 0, \forall x \in {X}\\
\vnu_{i,x} =& 0, \forall x \not \in {X}
\end{aligned}
\end{equation*}

Therefore, we are solving:

\begin{equation*}
\begin{aligned}
\min_{\vnu_i \in \Theta_{i}} &~~
\sum_{i} \sup_{\vxi^{O_i} \in \mathcal{M}} \left\{ m\vtheta_{i} \cdot \vxi^{O_i} 
- A^{*O_i}(\vxi^{O_i}) +
\vnu_i \cdot \vxi^{O_i} \right\}\\
s.t., &~~
\sum_{O_i: x \in {X}_{i}}
\vnu_{i,x} = 0, \forall x \in {X}\\
~~ &~~ \vnu_{i,x} = 0, \forall x \not \in {X}
\end{aligned}
\end{equation*}

Note the factor $m$ in front of $\vtheta_i$; it implies that we multiply
the weights in each subprogram by $m$ as well.
Then we can apply sub-gradient method on $\vnu_i$:

\begin{enumerate}
\item Initialize $\vnu_1^{(0)},...,\vnu_t^{(0)}$.
\item At step $k$ (start from 0):
\begin{enumerate}
\item For each operator $O_i$, solve the MLN program
consists of: 1) original rules in this operator, which
is characterized by $m\vtheta_{i}$; 2) additional priors on
each variables in $\mathcal{X}_{i}$, which is characterized
by $\vnu_i^{(k)}$.
\item Get the marginal inference results
  $\hat{\vxi_{i}^C}$.
\end{enumerate}
\item Update $\vnu_i^{(k+1)}$: \\$\vnu_{i,j}^{(k+1)} = \vnu_{i,j}^{(k)}
  - \lambda \left( \hat{\vxi_{i,j}^C} - \frac{\sum_{l: x_j \in
    {X}_{l}} \hat{\vxi_{l,j}^C}}{|\{l:x_j \in
  {X}_{l}\}|} \right)$
\end{enumerate}

\begin{example}
Consider the marginal inference on the case in Example \ref{example:packer}.
Similar to the example for MAP inference, we have
copies of random variables:
$\xi_{w,O_L}^C$, $\xi_{l,O_L}^C$ for $O_L$;
and $\xi_{w,O_C}^C$, $\xi_{l,O_C}^C$ for $O_C$.
We also have four $\nu$: $\nu_{w,O_L}$, $\nu_{l,O_L}$
for $O_L$; and $\nu_{w,O_C}$, $\nu_{l,O_C}$
for $O_C$. Assume we initialize each $\nu_-^{(0)}$ to $0$
at the first step.

We start by conducting marginal inference on $O_L$ and $O_C$ respectively.
In this case, $O_L$ will get the result:

\[
\begin{aligned}
\xi_{w,O_L}^C = 0.99\\
\xi_{l,O_L}^C = 0.01
\end{aligned}
\]

\noindent
while $O_C$ will get:

\[
\begin{aligned}
\xi_{w,O_C}^C = 0.5\\
\xi_{l,O_C}^C = 0.5
\end{aligned}
\]

\noindent
Assume the step size $\lambda = 0.5$. We can
update $\nu$ as:

\[
\begin{aligned}
\nu_{w,O_L}^{(1)} = & -0.12\\
\nu_{w,O_C}^{(1)} = & 0.12\\
~~ &\\
\nu_{l,O_L}^{(1)} = & 0.12\\
\nu_{l,O_C}^{(1)} = & -0.12\\
\end{aligned}
\]

Therefore, when we use these $\nu_-^{(1)}$ to conduct marginal inference
on $O_L$ and $O_C$, we are equivalantly adding

\begin{center}
{\small
\begin{tabular}{@{\hspace{-10pt}}ll@{\hspace{1pt}}l}
-0.12 & $\rel{label}(D, P_1, W)$ & $(r_{l1}')$\\
0.12 & $\rel{label}(D, P_1, L)$ & $(r_{l2}')$\\
\end{tabular}
}
\end{center}

\noindent
into $O_L$ and

\begin{center}
{\small
\centering
\begin{tabular}{@{\hspace{-10pt}}ll@{\hspace{1pt}}l}
0.12 &$\rel{label}(D, P_1, W)$ & $(r_{c1}')$\\
-0.12 &$\rel{label}(D, P_1, L)$ & $(r_{c2}')$\\
\end{tabular}
}
\end{center}

\noindent
into $O_C$.
Intuitively, one may interpret this procedure
as the information that ``$O_L$ prefers $\rel{label}(D, P_1, W)$ to be true''
being passed to $O_C$ via $r_{c}'$.
\end{example}

\section{Additional Details of System Implementation}

In this section, we provide additional details of the \felix system.
The first part of this section focuses on the compiler.
We prove some complexity results of property-annotation
used in the compiler and describe how to
apply static analysis techniques originally used in the
Datalog literature for data partitioning.
Then we describe the physical implementation for each
logical operator in the current prototype of \felix.
We also describe the cost model used for the
materialization trade-off.

\subsection{Compiler}

\subsubsection{Complexity Results}

\label{sec:app:comp}

In this section, we first prove the decidability of
the problem of annotating properties
for arbitrary Datalog programs. Then we
prove the $\Pi_2 \mathsf{P}$-completeness
of the problem of annotating $ \{REF, SYM\}$
given a Datalog program without recursion.

\paragraph*{Recursive Programs} If there is a single rule with query relation
$Q$ of the form $Q(x,y) <= Q1(x),Q2(y)$,
then that $\set{\text{REF},\text{SYM}}$ of $Q$ is decidable if and only if $Q1$ or $Q2$
is empty or $Q1 \equiv Q2$. We assume that $Q1$ and $Q2$ are
satisfiable. If there is an instance where $Q1(a)$ is true and $Q2$ is false
for all values. Then there is another world (with all fresh constants)
where $Q2$ is true (and does not return $a$). Thus, to check
$\text{REF}$ and $\text{SYM}$ for $Q$, we need to decide
equivalence of datalog queries. Equivalence of datalog queries is
undecidable~\cite[ch.~12]{abiteboul1995foundations}. Since containment
and boundedness for monadic datalog queries is decidable, a small
technical wrinkle is that while $Q1$ and $Q2$ are of arity one (monadic)
their bodies may contain other recursive (higher arity)
predicates.

\paragraph*{Complexity for Nonrecursive Program}

The above section assumes that we are given an arbitrary Datalog
program $\Gamma$.  In this section, we show that the problem of annotating
REF and SYM given a nonrecursive Datalog program is $\Pi_2
\mathsf{P}$-complete. We allow inequalities in the program.

We first prove the hardness. Similar to the above section,
we need to decide $Q1 \equiv Q2$. The difference is
that $Q1$ and $Q2$ do not have recursions.
Since our language allows us to express conjunctive
queries with inequality constraints, this established $\Pi_{2}
\mathsf{P}$ hardness~\cite{containment-jacm88}.


We now prove the membership in $\Pi_2 \mathsf{P}$. We first translate
the problem of property-annotation to the containment problem
of Datalog programs, which has been studied for decades
\cite{containment-stoc77,containment-jacm88} and
the complexity is in $\Pi_2 \mathsf{P}$ for Datalog programs
without recursions but with inequalities.
We will show that, even though
the rules for checking symmetric property is recursive, it
can be represented by a set of non-recursive rules, therefore
the classic results still hold.

We thus limit ourselves to non-recursive MLN programs. Given an MLN
program $\Gamma$ which is the union of conjunctive queries and a relation $Q$
to which we will annotate properties, all hard rules related to $Q$ can
be represented as:

\begin{equation}
\tag{$P_1$}
\begin{aligned}
Q() :-& G_1()\\
Q() :-& G_2()\\
...\\
Q() :-& G_n()
\end{aligned}
\end{equation}

\noindent
where each $G_i()$ contains a set of subgoals. To annotate whether
a property holds for the relation $Q()$, we test whether some
rules hold for all database instances $I$ generated by the above program $P_1$.
For example, for the symmetric property, we label $Q()$ as symmetric if and only if
$Q(x,y) => Q(y,x)$ holds.
We call this rule the {\em testing rule}.
Suppose the testing rule is $Q() :- T()$, we create a new
program:

\begin{equation}
\tag{$P_2$}
\begin{aligned}
Q() :-& G_1()\\
Q() :-& G_2()\\
...&\\
Q() :-& G_n()\\
Q() :-& T()
\end{aligned}
\end{equation}

Given a database $D$, let $P_1(D)$ be the result of applying program
$P_1$ to $D$ (using Datalog semantics). The testing rule
holds for all $P_1(D)$ if and only if
$\forall D$, $P_2(D) \subseteq P_1(D)$. In other words,
$P_2$ is contained by $P_1$ ($P_2 \subseteq P_1$).
For reflexive property, whose testing rule
is $Q(x,x) :- \mathcal{D}(x)$ (where $\mathcal{D}()$ is the
domain of $x$), both $P_1$ and $P_2$ are non-recursive
and the checking of containment is in $\Pi_2 \mathsf{P}$
\cite{containment-jacm88}.

We then consider the symmetric property, whose testing
rule is recursive.
This is difficult at first glance because the containment
of recursive Datalog program is undecidable.
However, for this special case, we can show it is much easier.
For the sake of simplicity, we consider
a simplified version of $P_1$ and $P_2$:

\begin{equation}
\tag{$P_1'$}
\begin{aligned}
Q(x,y) :- G(x,y,z)
\end{aligned}
\end{equation}

\begin{equation}
\tag{$P_2'$}
\begin{aligned}
Q(x,y) :-& G(x,y,z)\\
Q(x,y) :-& Q(y,x)
\end{aligned}
\end{equation}

We construct the following program:

\begin{equation}
\tag{$P_3$}
\begin{aligned}
Q(x,y) :-& G(x,y,z)\\
Q(x,y) :-& G(y,x,z)
\end{aligned}
\end{equation}

It is easy to show $P_2' = P_3$, therefore, we can equivalently
check whether $P_3 \subseteq P_1'$, which is
in $\Pi_2 \mathsf{P}$ since neither of the programs is recursive.

\subsubsection{Patterns Used by the Compiler}

\felix exploits a set of regular expressions for
property annotation. This set of regular expressions forms
a best-effort compiler, which is sound but not complete.
Table \ref{tab:easy:props} shows these patterns.
In \felix, a pattern consists of two components -- a
template and a boolean expression. A template is
a constraint on the ``shape'' of the formula. For example,
one template for SYM looks like $P_1(a,b) \vee !P_2(c,d)$,
which means we only consider rules whose disjunction
form contains exactly two binary predicates with opposite senses.
Rules that pass the template-matching are considered further
using the boolean expression. If one rule passes the
template-matching step, we can have a set of assignments
for each predicate $P$ and variable
$a,b,...$. The boolean expression is a first order logic
formula on the assignment. For example, the boolean expression
for the above template is $(a=d) \wedge (b=c) \wedge (P_1=P_2)$,
which means the assignment of $P_1$ and $P_2$ must be the same,
and the assignment of variables $a,b,c,d$ must satisfy
$(a=d) \wedge (b=c)$. If there is an assignment that satisfies
the boolean expression, we say this Datalog rule {\em matches} with this
pattern and will be annotated with corresponding labels.

\begin{table*}
\centering
\small
\begin{tabular}{l|l|l}
  \hline
  \multirow{2}{*}{Property} & \multicolumn{2}{c}{Pattern}    \\
                            & \multicolumn{1}{c}{Template} &\multicolumn{1}{c}{Condition} \\
  \hline
  \multirow{2}{*}{REF}                       & $P_1(a,b)$                 & $a=b$\\
    \cline{2-3}
                                              & $P_1(a,b) \vee !R_1(c) \vee !R_2(d)$ & $a=c,b=d,R_1=R_2,P_1\ne R_i$ \\
  \hline
  \multirow{2}{*}{SYM}                       & $P_1(a,b) \vee !P_2(c,d)$  & $a=d, b=c, P_1=P_2$\\
  \cline{2-3}
                                              & $P_1(a,b) \vee !R_1(c) \vee !R_2(d)$ & $a=c,b=d,R_1=R_2,P_1\ne R_i$ \\
  \hline
  \begin{minipage}[t]{1cm}TRN\\~\end{minipage}                       &
  \begin{minipage}[t]{3.5cm}$!P_1(a,b) \vee !P_2(c,d) \vee$$
  P_3(e,f)$\end{minipage} & \begin{minipage}[t]{10cm}$b = c$, $a = e$, $d =
  f$, $P_1=P_2=P_3$\end{minipage} \\
  \hline
  KEY                       & $!P_1(a,b) \vee !P_2(e,f) \vee [c=d]$
  & \begin{minipage}[t]{10cm}$a=e$, $b=c$, $d=f$,$ P_1=P_2$\end{minipage}
  \\
  \hline
  \multirow{2}{*}{NoREC}                     & $R_1() \vee \ldots \vee R_n()
  \vee P_1()$ & $P_1 \ne R_i$\\
  \cline{2-3}
   & $R_1() \vee \ldots \vee R_n()
  \vee !P_1()$ & $P_1 \ne R_i$\\
  \hline
  \multirow{8}{*}{TrRec}    & $P_1(a,b) \vee T(c,d) \vee
  P_2(e,f)$ & \begin{minipage}[t]{10cm}$b = c$, $d = f$, $a = e$, $P_1=P_2$,
    $T(c,d) = [d = c + x], x \ne 0$\end{minipage}
  \\
  \cline{2-3}
                            & $P_1(a,b) \vee T(c,d) \vee
  P_2(e,f)$ & \begin{minipage}[t]{10cm}$b = c$, $d = f$, $a = e$, $P_1=P_2$,
    $\forall (c,d) \in T, c \sqsubseteq d $\end{minipage}
  \\
    \cline{2-3}
  & $!P_1(a,b) \vee T(c,d) \vee
  P_2(e,f)$ & \begin{minipage}[t]{10cm}$b = c$, $d = f$, $a = e$, $P_1=P_2$,
    $T(c,d) = [d = c + x], x \ne 0$\end{minipage}
  \\
  \cline{2-3}
                            & $!P_1(a,b) \vee T(c,d) \vee
  P_2(e,f)$ & \begin{minipage}[t]{10cm}$b = c$, $d = f$, $a = e$, $P_1=P_2$,
    $\forall (c,d) \in T, c \sqsubseteq d $\end{minipage}
  \\
    \cline{2-3}
  & $P_1(a,b) \vee T(c,d) \vee
  !P_2(e,f)$ & \begin{minipage}[t]{10cm}$b = c$, $d = f$, $a = e$, $P_1=P_2$,
    $T(c,d) = [d = c + x], x \ne 0$\end{minipage}
  \\
  \cline{2-3}
                            & $P_1(a,b) \vee T(c,d) \vee
  !P_2(e,f)$ & \begin{minipage}[t]{10cm}$b = c$, $d = f$, $a = e$, $P_1=P_2$,
    $\forall (c,d) \in T, c \sqsubseteq d $\end{minipage}
  \\
    \cline{2-3}
  & $!P_1(a,b) \vee T(c,d) \vee
  !P_2(e,f)$ & \begin{minipage}[t]{10cm}$b = c$, $d = f$, $a = e$, $P_1=P_2$,
    $T(c,d) = [d = c + x], x \ne 0$\end{minipage}
  \\
  \cline{2-3}
                            & $!P_1(a,b) \vee T(c,d) \vee
  !P_2(e,f)$ & \begin{minipage}[t]{10cm}$b = c$, $d = f$, $a = e$, $P_1=P_2$,
    $\forall (c,d) \in T, c \sqsubseteq d $\end{minipage}
  \\
  \hline
\end{tabular}
\caption{Sufficient Conditions for Properties. All Patterns for REF,
  SYM, TRN, and KEY
are hard rules.}
\label{tab:easy:props}
\end{table*}

\subsubsection{Static Analysis for Data Partitioning}

Statistical inference can often be decomposed as independent subtasks
on different portions of the data. Take the
examples of classification in Section \ref{sec:specoper} for instance.
The inference of the query relation $\rel{winner}(team)$ is ``local''
to each $team$ constant (Assume $\rel{label}$ is the evidence relation).
In other words, deciding whether one $team$
is a winner does not rely on the decision of another team, $team'$,
in this classification subtask. Therefore, if there are a total of $n$
teams, we will have an opportunity to solve this subtask using
$n$ concurrent threads.
Another example is labeling, which is often local to small units of
sequences (e.g., sentences).

In \felix, we borrow ideas from the Datalog literature
\cite{seib:pods1991} that uses linear programming to perform static
analysis to decompose the data. \felix adopts the same algorithm of
Seib and Larsen~\cite{seib:pods1991}.

Consider an operator with query relation $\rel{R}(\bar{x})$.
Different instances of $\bar{x}$ may depend on each other
during inference.
For example, consider the rule

\[ \rel{R}(\bar{x}) <= \rel{R}(\bar{y}), \rel{T}(\bar{x}, \bar{y}). \]

Intuitively, all instances of $\bar{x}$ and $\bar{y}$ that appear in
the same rule cannot be solved independently since $\rel{R}(\bar{x})$
and $\rel{R}(\bar{y})$ are inter-dependent.
Such dependency relationships are transitive, and we want to compute them so that data partitioning wouldn't violate them. A straightforward approach is to ground all rules and then perform component detection on the resultant graph. But grounding tends to be very computationally demanding. A cheaper way is static analysis that looks at the rules only. Specifically,
one solution is to find a function $f_R(-)$
which has $f_R(\bar{x})=f_R(\bar{y})$ for all $\bar{x}$ and $\bar{y}$'s
that rely on each other.
As we rely on static analysis to find $f_R$, the above condition should
hold for all possible database instances.

Assuming each constant is encoded as an integer in \felix,
we may consider functions $f_R$ of the
form~\cite{seib:pods1991}:

\[ f_R(x_1,...,x_n) = \sum_i \lambda_i x_i \in \mathbb{N}, \]
where $\lambda_i$ are integer constants.

Following~\cite{seib:pods1991}, \felix uses linear programming to find
$\lambda_i$ such that $f_R(-)$ satisfy the above constraints.
Once we have such a partitioning function over the input, we can process
the data in parallel.
For example, if we want to run $N$ concurrent threads for $\rel{R}$,
we could assign all data satisfying

\[ f_R(x_1,...,x_n) \mod N = j \]

\noindent
to the $j^{th}$ thread.

\subsection{Operators Implementation}
\label{op-impl}
Recall that \felix selects physical implementations
for each logical operator to actually execute them.
In this section, we show a handful of physical
implementations for these operators.
Each of these physical implementations only works
for a subset of the operator configurations.
For cases not covered by these physical implementations,
we can always use \tuffy or Gauss-Seidel-Style implementations~\cite{tuffy-vldb11}.

\paragraph*{Using Logistic Regression for Classification Operators}

Consider a Classification operator with a query relation
$R(\underline{k},v)$, where $k$ is the key. Recall that each
possible value of $k$ corresponds to an independent classification
task. The (ground) rules of this operator are all non-recursive with
respect to $R$, and so can be grouped by value of $k$.
Specifically, for each value pair $\hat{k}$ and $\hat{v}$, define
\[
\begin{aligned}
\mathcal{R}_{\hat{k},\hat{v}} &=& \{r_i | r_i \mbox{ is violated when } R(\hat{k},\hat{v}) \mbox{ is true }\}\\
\mathcal{R}_{\hat{k},\bot} &=& \{r_i | r_i \mbox{ is violated when } \forall v\; R(\hat{k},\hat{v}) \mbox{ is false}\}
\end{aligned}
\]
and
\[
W_{\hat{k}, x} = \sum_{r_i\in\mathcal{R}_{\hat{k},x}}|w_i|
\]
which intuitively summarizes the penalty we have to pay for assigning $x$ for the key $\hat{k}$.

With the above notation, one can check that
\[
\Pr[R(\hat{k},x)\mbox{ is true}] = \frac{\exp\{-W_{\hat{k},x}\}}{\sum_{y}\exp\{-W_{\hat{k},y}\}},
\]
where both $x$ and $y$ range over the domain of $v$ plus $\bot$,
and $R(\hat{k},\bot)$ means $R(\hat{k},v)$ is false for all values of $v$.
This is implemented using SQL aggregation in a straightforward manner.

\eat{
We use the following approach to solve this case:

\begin{enumerate}
\item Initialize $h(k,v) = 0$.
\item For each grounded rule:
\begin{center}
{\small
\begin{tabular}{@{\hspace{-10pt}}ll@{\hspace{1pt}}l}
$w_i$ & $Q() => R(k,v)$\\
\end{tabular}
}
\end{center}
if this rule is violated, $h(k,v) \leftarrow h(k,v) - |w_i|$.
\item $MAP(k) = \arg \max_v h(k,v)$. $\Pr[R(k,\bar{v})] = \frac{\exp
    h(k,\bar{v})}{\sum_v \exp {h(k,v)}}$
\item $MAP(k)$ and $\Pr[R(k,\bar{v})]$ are the results of MAP inference
and marginal inference respectively.
\end{enumerate}

One can easily check this procedure will get exact inference results
for both MAP and marginal inference. In \felix, this process is
implemented by SQL aggregations.
}

\paragraph*{Using Conditional Random Field for Correlated Classification Operators}
The Labeling operator generalizes the Classification operator by
allowing tree-shaped correlations between the individual classification tasks.
For simplicity, assume that such tree-shaped correlation is actually a chain.
Specifically, suppose the possible values of $k$ are $k_1,\ldots,k_m$. Then
in addition to the ground rules as described in the previous paragraph,
we also have a set of recursive rules each containing $R(k_i,-)$ and $R(k_{i+1},-)$
for some $1\le i\le m-1$. Define
\[
\begin{aligned}
\mathcal{R}_{k_i,k_{i+1}}^B &=& \{r_i | r_i \mbox{ contains } R(k_i,-) \mbox{ and } R(k_{i+1},-)\}\\
W^B_{k_i,k_{i+1}}(v_i,v_{i+1}) &=& \sum_{r_i\in\mathcal{R}_{k_i,k_{i+1}}^B}\mbox{cost}_{r_i}(\{R(k_i,v_i),R(k_{i+1},v_{i+1})\}).
\end{aligned}
\]
Then it's easy to show that
\[
\Pr[\{R(k_i,v_i),1\le i\le m\}] \propto \exp\{-\sum_{1\le i\le m}W_{k_i,v_i} -
\sum_{1\le i\le m-1} W^B_{k_i,k_{i+1}}(v_i,v_{i+1})\},
\]
which is exactly a linear-chain CRF.

Again, \felix uses SQL to compute the above intermediate statistics, and then
resort to the Viterbi algorithm~\cite{lafferty2001conditional}
(for MAP inference) or the sum-product
algorithm~\cite{Wainwright:2008:GME:1523420} (for marginal inference).

\paragraph*{Using Correlation Clustering for Coreference Operators}
The Coref operator can be implemented using correlation clustering
~\cite{arasu2009large}. We show that the constant-approximation algorithm
for correlation clustering carries over to MLNs under some technical conditions.
Recall that correlation clustering essentially performs node partitioning
based on the edge weights in an undirected graph. We use the following example
to illustrate the direct connection between MLN rules and correlation clustering.

\begin{example}
Consider the following ground rules which are similar to those in Section \ref{sec:specoper}:

{\small
\begin{tabular}{ll}
$10$      & $\rel{inSameDoc}(P_1, P_2), \rel{sameString}(P_1,P_2) => \rel{coRef}(P_1, P_2)$\\
$5$      & $\rel{inSameDoc}(P_1, P_2), \rel{subString}(P_1,P_2) => \rel{coRef}(P_1, P_2)$\\
$5$      & $\rel{inSameDoc}(P_3, P_4), \rel{subString}(P_3,P_4) => \rel{coRef}(P_3, P_4)$
\end{tabular}
}
\end{example}

Assume $\rel{coRef}$ is the query relation in this Coreference operator.
We can construct the weighted graph as follows. The vertex
set is $V=\{P_1,P_2,P_3,P_4\}$. There are two edges with non-zero weight:
$(P_1, P_2)$ with weight 15 and $(P_3, P_4)$ with weight 5. Other edges
all have weight 0.
The following proposition shows that the correlation clustering algorithm
solves an equivalent optimization problem as the MAP inference in \mlns.
\begin{proposition}
\label{thm:coref-same-object}
Let $\Gamma(\bar{x}_i)$ be a part of $\Gamma$ corresponding to a coref subtask;
let $G_i$ be the correlation clustering problem transformed from $\Gamma(\bar{x}_i)$
using the above procedure. Then an optimal solution
to $G_i$ is also an optimal solution to $\Gamma(\bar{x}_i)$.
\end{proposition}

We implement Arasu et al.~\cite{arasu2009large} for correlation clustering.
The theorem below shows that, for a certain family of \mln programs,
the algorithm implemented in \felix actually performs approximate \mln inference.
\begin{theorem}
\label{thm:coref-approx}
Let $\Gamma(\bar{x}_i)$ be a coref subtask with rules generating a complete
graph where each edge has a weight of either $\pm\infty$ or $w$ s.t.
$m\le |w| \le M$ for some $m,M>0$. Then the correlation clustering algorithm
running on $\Gamma(\bar{x}_i)$ is a $\frac{3M}{m}$-approximation algorithm
in terms of the log-likelihood of the output world.
\end{theorem}
\begin{proof}
In Arasu et al.~\cite{arasu2009large}, it was shown that for the case $m=M$,
their algorithm achieves an approximation ratio of 3.
If we run the same algorithm, then in expectation the output violates no
more than $3\mathbf{OPT}$ edges, where $\mathbf{OPT}$ is the number of violated
edges in the optimal partition. Now with weighted edges, the optimal cost
is at least $m\mathbf{OPT}$, and the expected cost of the algorithm output
is at most $3M\mathbf{OPT}$. Thus, the same algorithm achieves $\frac{3M}{m}$ approximation.
\end{proof}

\subsection{Cost Model for Physical Optimization}
\label{app:cost-estimation}

The cost model in Section~\ref{sec:tradeoff} requires estimation of
the individual terms in $\textrm{ExecCost}$. There are three
components: (1) the materialization cost of each eager query, (2) the
cost of lazily evaluating the query in terms of the materialized views, and (3) the number
of times that the query will be executed $(t)$. We consider them in turn.

Computing (1), the subquery materialization cost $\textrm{Mat}(Q_i)$,
is straightforward by using PostgreSQL's EXPLAIN feature.
As is common for many RDBMSs, the unit of PostgreSQL's query
evaluation cost is not time, but instead an internal unit (roughly
proportional to the cost of 1 I/O). \felix performs all calculations in
this unit.

Computing (2), the cost of a single incremental evaluation, is more
involved: we do not have $Q_i$ actually materialized (and with indexes
built), so we cannot directly measure $\textrm{Inc}_{Q}(Q')$ using PostgreSQL. For
simplicity, consider a two-way decomposition of $Q$ into $Q_1$ and
$Q_2$. We consider two cases: (a) when $Q_2$ is estimated to be larger
than PostgreSQL assigned buffer, and (b) when $Q_2$ is smaller
(i.e. can fit in available memory).

To perform this estimation in case (a), \felix makes a
simplifying assumption that the $Q_i$ are joined together using
index-nested loop join (we will build the index when we actually
materialize the tables). Exploring clustering
opportunities for $Q_i$ is future work.

Then, we force the RDBMS to estimate the detailed costs of the plan
$\mathcal{P}: \sigma_{\bar{x}'=\bar{a}}(Q_1) \Join
\sigma_{\bar{x}'=\bar{a}}(Q_2)$, where $Q_1$ and $Q_2$ are views,
$\bar{x}'=\bar{a}$ is an assignment to the bound variables
$\bar{x}'\equiv\bar{x}^{\mathsf{b}}$ in $\bar{x}$. From the
detailed cost estimation, we extract the following quantities: (1)
$n_i$: be the number of tuples from subquery $\sigma_{\bar{x}}(Q_i)$;
(2) $n$: the number of tuples generated by $\mathcal{P}$.
We also estimate the cost $\alpha$ (in PostgreSQL's unit) of each I/O by
asking PostgreSQL to estimate the cost of selections on some existing tables.

Denote by $c'=\mathrm{Inc}_{Q}(Q')$ the cost (in PostgreSQL unit) of executing
$\sigma_{\bar{x}'=\bar{a}}(R_1) \Join \sigma_{\bar{x}'=\bar{a}}(R_2)$,
where $R_i$ is the materialized table of $Q_i$ with proper indexes built.
Without loss of generality, assume $n_1 < n_2$ and that $n_1$ is small enough so
that $\Join$ in the above query is executed using nested loop join.
On average, for each of the estimated $n_1$ tuples in $\sigma_{\bar{x}}(R_1)$, there is
one index access to $R_2$, and
$\lceil\frac{n}{n_1}\rceil$ tuples in $\sigma_{\bar{x}}(R_2)$ that can be joined;
assume each of the $\lceil\frac{n}{n_1}\rceil$ tuples from $R_2$ requires
one disk page I/O. Thus, there are $n_1\lceil \frac{n}{n_1} \rceil$ disk accesses to
retrieve the tuples from $R_2$, and
\begin{equation}
c' = \alpha n_1\left[\lceil \frac{n}{n_1} \rceil + \log |Q_2|\right]
\end{equation}
where we use $\log|Q_2|$ as the cost of one index access to $R_2$ (height of a B-tree).
Now both $c'=\mathrm{Inc}_{Q}(Q')$ and $\mathrm{Mat}(Q_i)$ are in the unit of PostgreSQL cost,
we can sum them together, and compare with the estimation on other materialization plans.

In case (b), when $Q_2$ can fit in memory, we found that the
above estimation tends to be too conservative -- many accesses to
$Q_2$ are cache hits whereas the model above still counts the accesses
into disk I/O. To compensate for this difference, we multiply $c'$
(derived above) with a fudge factor $\beta<1$. Intuitively, we choose
$\beta$ as the ratio of accessing a page in main memory versus
accessing a page on disk. We empirically determine $\beta$.

Component (3) is the factor $t$, which is dependent on the
statistical operator. However, we can often derive an estimation
method from the algorithm inside the operator.  For example, for
the algorithm
in~\cite{arasu2009large}, the number of requests to an input
data movement operator can be estimated by the total number
of mentions (using COUNT) divided by the expected average node degree.

\section{Additional Experiments}

\subsection{Additional Experiments of High-level Scalability and Quality}
\label{app:exp-quality}

We describe the detailed methodology in our experiments on
the Enron-R, DBLife, and NFL datasets.

\paragraph*{Enron-R}
The \mln program for Enron-R was based on the rules obtained from
related publications on rule-based information 
extraction~\cite{michelakis2009uncertainty,liu2010automatic}.
These rules (i.e., ``Rule Set 1'' in Figure~\ref{fig:pr-all3}) use 
dictionaries for person name extraction, and regular expressions
for phone number extraction.
To extract person-phone relationships,
a fixed window size is used to identify person-phone co-occurrences.
We vary this window size to produce a precision-recall curve of this
rule-based approach.

The \mln program used by \felix,\tuffy,and \alc replaces the above rules' relation extraction
part (using the same entity extraction results) with a statistical counter-part: 
Instead of fixed window sizes,
this program uses \mln rule weights to encode the strength of co-occurrence
and thereby confidence in person-phone relationships. In addition, we write
soft constraints such as ``{\it a phone number cannot be associated
with too many
persons}.'' We add in a set of coreference rules to perform
person coref. We run \alc, \tuffy and \felix on this program.

\paragraph*{DBLife}
The \mln program for DBLife was based on the rules in \cimple~\cite{dblife2007},
which identifies person and organization mentions
using dictionaries with regular expression variations (e.g., abbreviations, titles).
In case of an ambiguous mention such as ``J. Smith'', \cimple binds it
to an arbitrary name in its dictionary that is compatible (e.g., ``John Smith'').
\cimple then uses a proximity-based formula to
translate person-organization co-occurrences into ranked affiliation
tuples. These form ``Rule Set 2'' as in Figure~\ref{fig:pr-all3}.

The \mln program is constructed as follows. 
We first extract entities from the corpus. We perform part-of-speech
tagging~\cite{treetaggerschmid1999improvements} on the raw text, and
then identify possible person/organization names using simple
heuristics (e.g., common person name dictionaries and keywords such as
``University'').  To handle noise in the entity extraction results,
our MLN program performs both affiliation extraction and coref
resolution using ideas similar to Figure~\ref{fig:mlns}.

\paragraph*{NFL} On the NFL dataset, we extract winner-loser pairs.
There are 1,100 sports news articles in the corpus.
We obtain ground truth of game results from the web.
As the baseline solution, we use 610 of the articles together with ground truth
to train a CRF model that tags each token in the text
as either WINNER, LOSER, or OTHER. We then apply this CRF model on the
remaining 500 articles to generate probabilistic tagging of the tokens.
Those 500 articles report on a different season of NFL games than the training articles,
and we have ground truth on game results (in the form of winner-loser-date triples).
We take the publication dates of the articles and align them to game dates.

The MLN program on NFL consists of two parts.
The first part contains MLN rules encoding the CRF model
for winner/loser team mention extraction.
The second part is adapted from the rules
developed by a research team in the Machine Reading project.
Those rules model simple domain knowledge such as ``a winner cannot be a loser
on the same day'' and ``a team cannot win twice on the same day.''
We also add in coreference of the team mentions.

\begin{table*}[thb]\centering\small
  \begin{tabular}{|l||r|r|r|r|r|}
      \hline
       & \textbf{Coref} & \textbf{Labeling} & \textbf{Classification}
       & \textbf{MLN Inference} \\
      \hline
    \textbf{Enron-R} & 1/1 & 0/0 & 0/0 & 1/1  \\
    \textbf{DBLife} & 2/2 & 0/0 & 1/1 & 0/0   \\
    \textbf{NFL}   & 1/1 & 1/1 & 0/0 & 1/1   \\
    \hline
    \hline
    \textbf{Program1} & 0/0 & 1/1 & 0/0 & 0/0 \\
    \textbf{Program2} & 0/0 & 0/0 & 37/37 & 0/0 \\
    \textbf{Program3} & 0/0 & 0/1 & 0/0 & 1/1 \\
    \hline
  \end{tabular}
  \caption{Specialized Operators Discovered by Felix's Compiler}
  \label{tab:cover}
\end{table*}

\begin{figure}[t]
\centering
  \includegraphics[width=0.30\textwidth]{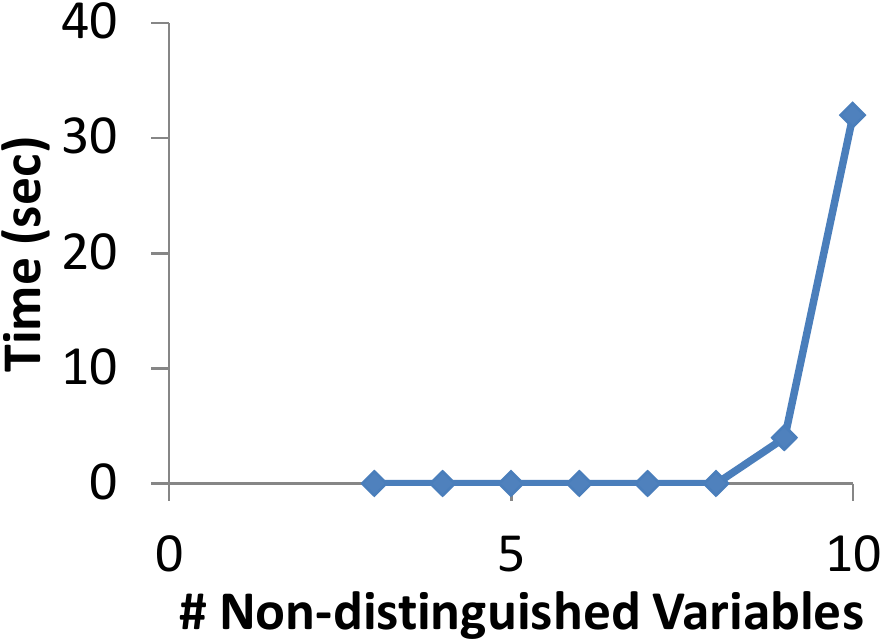}
  \caption{Performance of $\Pi_2 \mathsf{P}$-complete Algorithms
for Non-recursive Programs}
  \label{fig:pi2p}
\end{figure}

\begin{figure}[t]
\centering
  \includegraphics[width=0.30\textwidth]{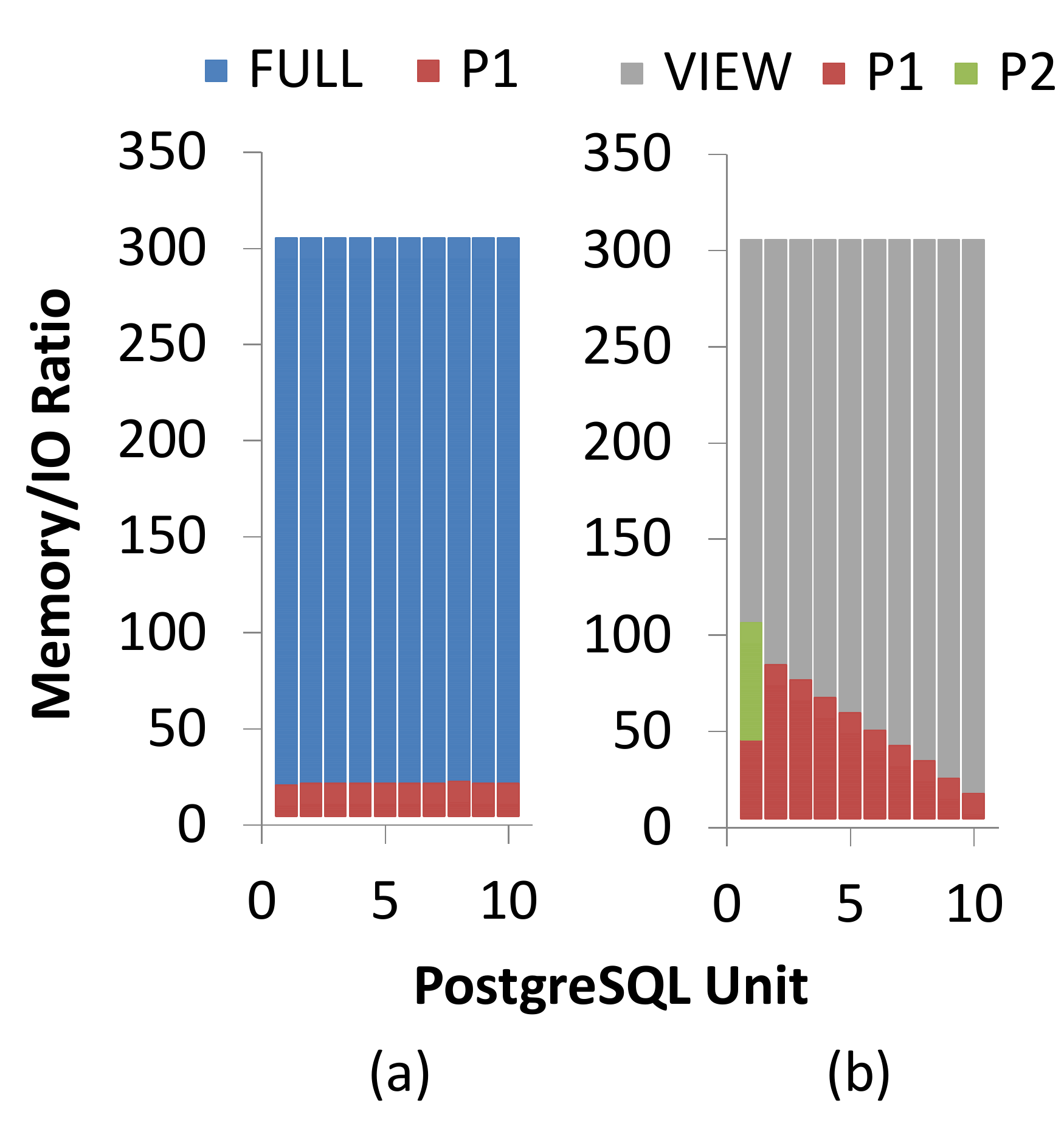}
  \caption{Plan diagram of \felix's Cost Optimizer}
  \label{fig:plan-diagram}
\end{figure}

\begin{figure}[t]
\centering
  \includegraphics[width=0.30\textwidth]{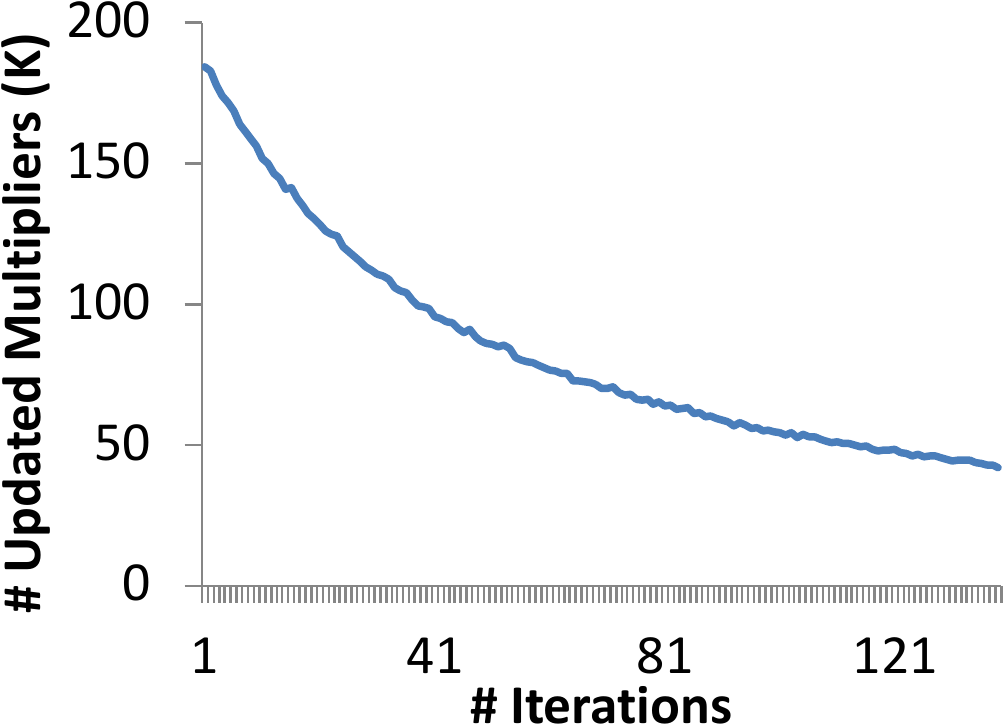}
  \caption{Convergence of Dual Decomposition}
  \label{fig:convergence}
\end{figure}

\subsection{Coverage of the Compiler}

Since discovering subtasks as operators is crucial to \felix's
scalability, in this section we test \felix's compiler.
We first evaluate the heuristics we are using
for discovering statistical operators given an MLN program. We
then evaluate the performance of the $\Pi_2 \mathsf{P}$-complete
algorithm to discovering REF and SYM in non-recursive programs.

\paragraph*{Using Heuristics for Arbitrary MLN Programs}

While \felix's compiler can discover all Coref, Labeling, and Classification
operators in all programs used in our experiments,
we are also interested in how many operators \felix can discover
from other programs. To test this, we download the programs
that are available on \alc's Web site
\footnote{\url{http://alchemy.cs.washington.edu/mlns/}} and
manually label operators in these programs.
We manually label a set of rules as an operator if this set of rules
follows our definition of statistical operators.

We then
run \felix's compiler on these programs and compare the logical plans produced by \felix
with our manual labels. We list all programs with manually labeled
operators in Table~\ref{tab:cover}. The $x/y$ in each cell of Table~
\ref{tab:cover} means that, among $y$ manually labeled operators,
Felix's compiler discovers $x$ of them.

We can see from Table~\ref{tab:cover} that \felix's compiler
works well for the programs used in our experiment. Also,
\felix works well on discovering classification and labeling
operators in \alc's programs. This implies the set of heuristic rules we are using,
although not complete, indeed encodes some popular
patterns users may use in real world applications.
Although some of \alc's programs encode
coreference resolution tasks, none of them were
labeled as coreference operator. This is because
none of these programs explicitly declares the symmetric constraints
as hard rules. Therefore, the set of possible worlds decided
by the MLN program is different from those decided by
the typical ``partitioning''-based semantics of coreference operators.
How to detect and efficiently implement these ``soft-coref''
is an interesting topic for future work.

\paragraph*{Performance of $\Pi_2 \mathsf{P}$-complete Algorithm
for Non-recursive Programs}

In Section \ref{sec:compilation} and Section \ref{sec:app:comp}
we show that there are $\Pi_2 \mathsf{P}$-complete algorithms
for annotating REF and SYM properties. \felix implements them.
As the intractability is actually inherent in the number of
non-distinguished variables, which is usually small, we are interested
in understanding the performance of these algorithms.

We start from one of the longest rules found in \alc's Web site
which can be annotated as SYM. This rule has 3 non-distinguished
variables. We then add more non-distinguished variables
and plot the time used for each setting (Figure~\ref{fig:pi2p}).
We can see that \felix uses less than 1 second to annotate the original
rule, but exponentially more time when the number
of non-distinguished variables grows to 10. This is not surprising
due to the exponential complexity of this algorithm. Another
interesting conclusion we can draw from Figure~\ref{fig:pi2p} is
that, as long as the number of non-distinguished variables
is less than 10 (which is usually the case in our programs),
\felix performs reasonably efficiently.

\subsection{Stability of Cost Estimator}

In our previous experiments we show that the plan generated by \felix's
cost optimizer contributes to the scalability of \felix. As the
optimizer needs to estimate several parameters before
performing any predictions, we are interested in the sensitivity of
our current optimizer to the estimation errors of these parameters.

The only two parameters used in \felix's optimizer are 1) the cost
(in PostgreSQL's unit) of fetching one page from the disk and 2)
the ratio of the speed between fetching one page from the memory
and fetching one page from the disk. We test all combined settings
of these two parameters ($\pm 100\%$ of the estimated value)
and draw the plan diagram of two queries in
Figure \ref{fig:plan-diagram}. We represent different execution plans
with different colors. For each point $(x,y)$ in the plan diagram, the
color of that point represents which execution plan the compiler
chooses if the PostgreSQL's unit equals $x$ and memory/IO
ratio equals $y$.

For those queries not shown in Figure \ref{fig:plan-diagram},
\felix produces the same plan for each tested parameter combination.
For queries shown in Figure \ref{fig:plan-diagram}, we can see
\felix is robust for parameter mis-estimation. Actually,
all the plans shown in Figure \ref{fig:plan-diagram} are close to optimal,
which implies that in our experiments \felix's cost optimizer avoids
the selection of ``extremely bad'' plans even under serious
mis-estimation of parameters.

\subsection{Convergence of Dual Decomposition}

\felix implements an iterative approach for dual
decomposition. One immediate question is
{\it how many iterations do we need before the algorithm converges?}.

To gain some intuitions, we run \felix on the DBLife
\footnote{Similar phenomena occur in the NFL dataset as well.}
data set for a relative long time and record the number of updated
Lagrangian multipliers of each iteration. We use constant step size $\lambda=0.9$.
As shown in Figure \ref{fig:convergence}, even after more than 130 iterations,
the Lagrangian multipliers are still under heavy updates.
However, on the ENRON-R dataset, we observed that the whole process
converges after the first several iterations!
This implies that the convergence of our operator-based
framework depends on the underlying MLN program
and the size of the input data.
It is interesting to see how different techniques on dual decomposition
and gradient methods can alleviate this convergence issue,
which we leave as future work.

Fortunately, we empirically find that in all of our experiments,
taking the result from the first several iterations is often a reasonable
trade-off between time and quality -- all P/R curves in the previous experiments are generated
by taking the last iteration within 3000 seconds and we already
get significant improvements compared to baseline solutions. In \felix, to allow
users to directly trade-off between quality and performance,
we provide two modes: 1) Only run the first iteration
and flush the result immediately; and 2) Run the number of iterations
specified by the user. It is an interesting direction to
explore the possibility of automatically
selecting parameters for dual decomposition.

\eat{
\bibliographystyleapp{abbrv}
\bibliographyapp{felixapp}
}

\end{document}